\documentclass{article}

\usepackage{PRIMEarxiv}

\usepackage[utf8]{inputenc} 
\usepackage[T1]{fontenc}    
\usepackage{hyperref}       
\usepackage{url}            
\usepackage{booktabs}       
\usepackage{amsfonts}       
\usepackage{nicefrac}       
\usepackage{microtype}      
\usepackage{lipsum}
\usepackage{fancyhdr}       
\usepackage{graphicx}       
\graphicspath{{media/}}     
\usepackage{natbib}
\usepackage{enumitem}
\usepackage[skip=2pt,font=scriptsize]{subcaption}
\usepackage{amsmath,amssymb}
\usepackage{amsthm} 
\usepackage[thinlines]{easytable}
\usepackage{microtype}
\usepackage{array, booktabs}
\usepackage{bbm}

\DeclareMathOperator*{\argmax}{argmax} 

\newtheorem{insight}{Insight}
\newtheorem{proposition}{Proposition}
\newtheorem{theorem}{Theorem}
\newtheorem{lemma}{Lemma}
\newtheorem{corollary}{Corollary}

\title{To Start Up a Start-Up---Embedding Strategic Demand Development in Operational On-Demand Fulfillment via Reinforcement Learning with Information Shaping
}

\author{
  Xinwei Chen \\
  School of Finance and Operations Management \\
  The University of Tulsa \\
  Tulsa, US \\
  \texttt{xinwei-chen@utulsa.edu} \\
   \And
  Marlin W. Ulmer \\
  Otto von Guericke Universität Magdeburg \\
  Magdeburg, Germany \\
  \texttt{marlin.ulmer@ovgu.de} \\
  \AND
   Barrett W. Thomas \\
   Department of Business Analytics \\
   The University of Iowa \\
   Iowa City, US \\
   \texttt{barrett-thomas@uiowa.edu} \\
}

\begin{document}
\maketitle

\begin{abstract}
The last few years have witnessed rapid growth in the on-demand delivery market, with many start-ups entering the field. However, not all of these start-ups have succeeded due to various reasons, among others, not being able to establish a large enough customer base. In this paper, we address this problem that many on-demand transportation start-ups face: how to establish themselves in a new market. When starting, such companies often have limited fleet resources to serve demand across a city. Depending on the use of the fleet, varying service quality is observed in different areas of the city, and in turn, the service quality impacts the respective growth of demand in each area. Thus, operational fulfillment decisions drive the longer-term demand development. To integrate strategic demand development into real-time fulfillment operations, we propose a two-step approach. First, we derive analytical insights into optimal allocation decisions for a stylized problem. Second, we use these insights to shape the training data of a reinforcement learning strategy for operational real-time fulfillment. Our experiments demonstrate that combining operational efficiency with long-term strategic planning is highly advantageous. Further, we show that the careful shaping of training data is essential for the successful development of demand. 
\end{abstract}


\section{Introduction}

Recent years have seen the rise of many start-ups in urban transportation and delivery. Some start-ups survive the critical first years and turn into successful companies. Examples are GrubHub for restaurant meal delivery, GoPuff for delivery of convenience store items, Instacart for grocery delivery, Uber for passenger transportation, and Bird for e-scooter rentals. However, besides the few success stories, there are hundreds of start-ups that eventually failed. For instance, after a successful venture capital funding round in 2021, the German instant delivery start-up Gorillas became a start-up ``unicorn" with a value of over a billion US dollars. A few years later, the company is on the brink of bankruptcy. Similar stories can be told for many start-ups in instant delivery, with their creative names (``Zume", ``Jokr", ``Alpakas", ``MilkRun") already forgotten. 
The reasons behind the fall of so many start-ups are manifold. Their business model might not be sustainable, or the competition is too fierce. In addition, another crucial reason could be that companies fail to grow their business early by satisfying and expanding their customer base \citep{picken2017startup,capital2024}. Especially in the early start-up phase, the available resources for such companies are often limited. For example, the German ride-sharing start-up MOIA had only 100 vehicles available when launching its service in Hamburg, Germany, a city of nearly two million people. At the same time, there is a global shortage of drivers in the urban delivery market, leading to limited delivery resources, failed deliveries, and customer dissatisfaction, as experienced by companies such as Domino's \citep{Haddon_2022}.  

So, how should delivery companies use their delivery resources to establish and grow their customer base? The question may seem trivial to giant companies as they have adequate delivery resources. However, there still exist a large number of smaller-sized start-up companies that struggle to meet all the demand, particularly since the COVID-19 pandemic \citep{Mansfield_2022,SupplyChainBrain_2023}. For a delivery company with limited resources, should it use its resources to serve and grow demand all over the city, risking limited service availability, compromised quality, and slow growth? Or should it prioritize service in certain parts of the city to grow demand fast but likely neglect customers in other parts? As we show in this paper, the answer is not simple. First, managing daily operations of urban delivery and transportation services is inherently challenging. Real-time decisions must be made about whom to offer service to and how to dynamically route the service fleet in reaction to current needs and in anticipation of future demand. Consequently, there is a surge in papers focused on optimizing the dispatching and routing of service fleets \citep{liu2021time,stroh2022tactical,zhang2023routing}. The main body of the existing literature treats demand distributions as exogenous and, for the given distributions, optimizes daily operations, such as maximizing the demand served in the city for that day. Another stream in the literature studies the strategic control of resources and demand \citep{perera2020retail,deng2021urban,chen2022food,he2022profit}. This area addresses on-demand challenges with a focus on higher-level supply and demand management, rather than operational fulfillment optimizations. However, when looking at the early phases of start-ups, the service is new and demand is impacted by who is served. Demand in a region can increase over time if enough resources are dedicated to keeping customer satisfaction high, and at the same time, demand in a region might decrease if service quality is insufficient \citep{y2004drives,cheung2009impact,lo2012consumer}. 
Even though the two parts, effective ``intra-day" operations and strategic ``inter-day" control, are strongly intertwined by endogeneity of demand, an integrated consideration has not been studied in the literature yet.

In this paper, we study the on-demand, same-day delivery problem with endogenous demand. We seek to manage daily operations while accounting for the impact on future demand. To do so, we begin by seeking to understand the end that we are seeking. We investigate a stylized model of the inter-day control problem for two service regions. In this stylized, two-region problem, we prove that the optimal strategy should either distribute resources equally to both service regions or focus on one region exclusively. Which strategy to choose depends on the potential for demand growth in the regions. We then integrate the insights of the proof into the training of a policy for daily ``intra-day" operations. This policy uses reinforcement learning (RL) to decide which stochastically requesting customers to serve and how to dynamically route the service fleet accordingly. To incorporate the insights into strategic ``inter-day" demand control, we actively guide the training data of the RL policy to nudge our policy toward a desired outcome. We call this adaptation of training data \textit{information shaping}, to the best of our knowledge, a new concept in the field of RL. We conduct a comprehensive computational study. We derive the following managerial insights:

\begin{itemize}[leftmargin=*]
    \item When executed effectively, strategic demand development via inter-period anticipation can be a powerful tool to enhance business and grow revenue in the long run. 
    
    \item Effective strategic demand development improves service levels and revenue during the planning horizon and leads to a larger customer base for future operations.
    
    \item When regions are similar and demand per region is limited, it is worth investing in all regions even if initial demand in one region is small. If demand per region is unlimited, focusing on the region with high initial demand might be advantageous.
    
    \item The spatial distribution of customers plays an important role in the decision on how to invest fleet resources. Here, it is usually beneficial to focus on conveniently located regions.
          
    \item The length of the time horizon impacts the effectiveness of strategic demand development decisions.
            
    \item Prioritizing regions might lead to more revenue in the planning horizon compared to treating all regions equally, but it also results in a very unbalanced final demand structure. 
\end{itemize}

We make the following problem- and method-oriented contributions:

\paragraph{Problem.} We are among the first to integrate long-term objectives into on-demand delivery operations. We present analytical insights into the general problem structure and use these insights in the design of an effective decision policy. We show the value of the policy in a comprehensive computational study and present important managerial insights. 

\paragraph{Methodology.} To the best of our knowledge, we are the first to present information shaping in a reinforcement learning context. We illustrate its general functionality and show its effectiveness in a problem of high complexity.

The paper is organized as follows. Section~\ref{sec:lit} reviews the existing literature on demand evolution, same-day delivery, demand management in on-demand delivery problems, and reinforcement learning. Section~\ref{sec:problem} presents a sequential decision process model of the problem. In Section~\ref{sec:methodology}, we discuss the motivation of our approach, formally introduce the concept of information shaping, analyze a stylized inter-day problem, and use the insights to design our information shaping RL-policy. Section~\ref{sec:experiments} provides an extensive experimental study and managerial implications. Section~\ref{sec:conclusions} closes the paper with conclusions and future work.

\section{Literature Review}\label{sec:lit}

Our work is novel in several aspects. First, we consider the evolution of customer demand for on-demand delivery. Second, we integrate demand evolution into the problem model and maximize the company's service level in the long term. Finally, we propose novel training and implementation strategies for reinforcement learning. To this end, Section~\ref{lit:model} reviews the literature on demand evolution and different demand models used in last-mile logistics and operations management. In Section~\ref{lit:objective}, we review the existing literature on on-demand delivery, which usually disregards demand evolution. Section~\ref{lit:management} reviews the main approaches to managing customer demand in the literature on last-mile delivery. Section~\ref{lit:rl} presents the techniques used to modify information in machine learning and reinforcement learning.

\subsection{Demand Evolution}\label{lit:model}


In the majority of work with stochastic demand, it is assumed that the demand distribution is exogenous and static, i.e., customer demand follows a certain distribution with a deterministic and constant expected value. However, in practice, demand may be endogenous because individual customers' previous shopping experience affects their future purchase decisions and thus impacts the overall demand for the service \citep{y2004drives}. For example, customers exchange their shopping experience with their relatives, friends, neighbors, etc., through word of mouth \citep{arndt1967word}. Research shows that satisfaction significantly impacts the volume of word of mouth \citep{fainmesser2021ratings} and that word of mouth impacts customers' purchase decisions \citep{cheung2009impact,lo2012consumer}. The increasing use of online social media nowadays further strengthens this impact \citep{prasad2017social}. Our problem takes into consideration this impact of service level on customer demand. Specifically, we assume that the demand for service from a neighborhood may increase as the company makes more deliveries to the neighborhood, and vice versa. Thus, how a company allocates limited delivery resources to different neighborhoods today impacts future demand. 

When it comes to optimizing with demand evolution, the most closely related work is that of \citet{he2022profit}. In that work, the authors consider an on-demand delivery problem in a hybrid-workforce environment, where a traditional service provider cooperates with a platform and can transfer part of the demand to the platform. The objective is to maximize the total profits the service provider and platform receive over the long run. They model the service provider's demand as being linearly dependent on the demand in the previous period and \textit{market thickness}. The authors define market thickness as the available freelancers on the platform, which is analogous to \textit{inventory level} to be discussed next. Our work also seeks to maximize utility in the long run but differs in the complexity of the problem. \citet{he2022profit} model their problem as a dynamic matching problem, and the decisions are concerned with how many orders to transfer to the platform each day. While our problem is a dynamic vehicle routing problem---the decisions involve an acceptance decision for each customer request as well as the routing decisions throughout the day. Hence, the complexity of our problem is substantially amplified, which demands a different solution strategy that can balance complexity and real-time computation. In their work, \citet{he2022profit} also model the platform's demand as being linearly dependent on the number of available freelancers, characterized by a random variable. The authors derive managerial insights regarding the optimal level of collaboration between the service provider and the platform for different combinations of the random variable's mean and standard deviation. In our work, we consider various demand models and offer insights into the optimal allocation of delivery resources with different geography and varying potential for demand growth in different neighborhoods. Another work that is related to on-demand delivery but considers the evolution of workforce is \citet{luy2023strategic}. The authors focus on managing the development of different types of crowdsourced drivers while minimizing the total costs of wages and operational costs. Different from our work, the changes in the number of drivers are modeled by certain probability distributions, not dependent on service levels. 



In the literature on operations and supply chain management, there are works available related to inventory- or service-dependent demand, which do not have a routing component. \citet{deng2014statistical} study a multi-period newsvendor problem in which customers may leave the system after encountering a stockout. They investigate four different problem settings and show that a state-independent stockout penalty cost should be used when demand can be learned from historical data. Although the problems are different, similar to leaving after a stockout, we assume that customers in a certain region may stop requesting services if their request was not fulfilled. \citet{qian2014market} studies supplier selection in a market-based environment in which demand is linearly dependent on attributes including service level. The author shows that, when selecting suppliers, companies should align their operational characteristics, such as costs and delivery time, with market characteristics including demand and customer sensitivity to maximize profitability. Similar to \citet{qian2014market}, we also consider service level as a decision variable and assume that demand depends on service level. Through a field experiment at a supplier, \citet{craig2016impact} investigate the impact of a supplier's inventory service level on retailers' demand. The results show that increases in historical fill rate are associated with statistically significant and managerially substantial increases in current retailer orders, and this impact is on the demand rather than just sales. \citet{ccavdar2023word} investigate the optimal shipment policy for retailers when customer demand can be impacted by the perceived service quality. In their problem, the authors measure the perceived service quality by delivery time, and the decisions are determining the length of each shipment cycle. They study various shipment policies based on whether the retailer is aware of demand evolution and show that market size and customer sensitivity are important factors that determine long-term demand. In this paper, although the problem is different, we assume a demand model similar to the work mentioned above. Specifically, the demand from customers increases as the company's service level increases, and vice versa. There is also a stream of research that focuses on using historical data to learn inventory decisions when demand can change, such as \citet{chen2021data} and \citet{xiong2022data}. 

\subsection{Objectives in Same-Day Delivery Problems}\label{lit:objective}


A same-say delivery (SDD) problem is a dynamic and stochastic vehicle routing problem, as decisions are made throughout the day and order information is not known in advance. Due to the fast growth of e-commerce in the recent decade, there is increasing literature on SDD. \citet{boysen2021last} provide a comprehensive review of last-mile delivery problems including SDD from the perspective of operations research. 

In SDD, due to the restrictions on drivers' working hours, service providers usually operate for a limited number of hours a day. Because of these separate operational periods from day to day, the majority of the literature maximizes a notion of daily utility \citep{klapp31,klapp32,klapp2020request,drones,ulmer45,cosmi2019scheduling,cote2021dynamic,liu_2019,voccia,dayarian2020same,dayarian2020crowdshipping,schubert2020same,ahamed2020deep,ulmer2020dynamic, bracher2021learning,Jahanshahi,deepQ,stroh2022tactical,banerjee2022fleet,banerjee2022has,chen2023same,auad2024dynamic}. This daily utility is commonly in the form of service level, such as the number of customers served or the experienced delay of customers. In this research, although our methods and insights can be applied to other objectives, we seek to maximize the number of customers served. However, unlike the existing literature, we do so with consideration of demand evolution. Therefore, our ultimate goal is to maximize the expected number of services provided over a horizon that is much longer than a day (e.g., two years). Despite having different objectives, \citet{ahamed2020deep,Jahanshahi,deepQ,chen2023same} among the work mentioned above also use a RL approach in their work. In our experiments, one of the benchmark policies uses the classical RL approach similar to those methods.  

\subsection{Demand Management with Vehicle Routing}\label{lit:management}

Demand management has been extensively studied in operations management and supply chain management. However, it has only recently garnered attention in the realm of transportation and logistics, particularly within the last decade. \citet{fleckenstein2023recent} provide a comprehensive review of demand management approaches for on-demand delivery problems. The authors review work on demand management in attended home delivery, same-day delivery, and mobility-on-demand. \citet{wassmuth2023demand} review the demand management in the literature particularly on attended home delivery.

The existing literature manages customer demand mainly through acceptance decisions, assortment optimization, or dynamic pricing. Upon receiving a request or a batch of requests, the service provider can decide which customers to accept for service according to the service provider's capacity, profitability, and the resource reservation for the future \citep{hosni2014shared,xu2018large,ulmer2018budgeting,holler2019deep,kullman2022dynamic,giallombardo2022profit}. Note that, although using acceptance decisions in their problems, these papers optimize in a single-day context. Assortment optimization refers to managing demand when multiple products or fulfillment options are available \citep{atasoy2015concept,mackert2019choice,gallego2019revenue,lang2021anticipative,lang2021multi}. Dynamic pricing seeks to price the service dynamically \citep{campbell2006incentive,yang2016choice,vinsensius2020dynamic,al2020approximate,ulmer2020dynamic}.

Unlike our work, the work mentioned above does not consider demand development in the longer term, while we manage customer demand by looking at a much longer horizon. In this paper, we manage demand through acceptance decisions. Specifically, for each customer request received, the service provider needs to decide whether to offer the service with consideration of delivery capacity, delivery deadlines, and reservation of delivery resources for future requests. Due to the demand evolution over days, these acceptance decisions do not only manage the demand in the day but also impact customer demand in the future. 


\subsection{Information Shaping in Machine Learning}\label{lit:rl}


In this paper, we propose information shaping for RL. Information shaping is defined by controlling the selection of data, in our case demand scenarios, used to train the policy. This approach differs from existing methods in the RL literature that seek to guide learning by manipulating the Bellman equation~(\ref{eq:bellman}) in different ways, including modifying the decision space (action shaping), the reward function (reward shaping), and the expected cost/reward-to-go \citep{hildebrandt2023reinforcement}. In our problem, a single, demand-sensitive policy is used on a horizon consisting of many days, and customer demand evolves over the days. Therefore, information shaping controls the demand distributions used to generate training data for the RL. Our goal in shaping information is to guide the policy toward the theoretically optimal solution while still ensuring effective decision-making in individual days. To the best of our knowledge, no similar RL-approaches exist.

Information shaping offers a way to incorporate human domain knowledge and general analytical insights into the machine's learning process. In the existing RL literature, another way to embed such knowledge is through reward shaping. Reward shaping encourages intended decisions by assigning different rewards to decisions \citep{ng1999policy,laud2002reinforcement,grzes2017reward,eschmann2021reward}. However, as we will show later, policies that use reward shaping are not effective for our problem. The reason is that, when demand can change over time, simply modifying the reward is not capable of actively controlling how demand evolves, and thus the learned policies cannot effectively cope with demand evolution. 

When the reward function is not easy to define, the existing literature handles the issue with human-in-the-loop RL (HLRL). HLRL initializes training using a supervised model to predict whether a response is good or bad using human labeling \citep{retzlaff2024human}. HLRL is, for example, used in the training of large language models where humans can evaluate the output's quality \citep{economist2024}. For our problem, HLRL is difficult to implement because we do not know whether a policy is good or not in the long run.

Another option is to control the decisions available to the RL. This is known as action shaping \citep{kanervisto2020action} or action forging \citep{zadeh2024explainable}. The method is used to either avoid obviously ineffective decisions or to create policies that would be ``explainable" to a human user. 

While we are not aware of other literature that incorporates information shaping as we do in this paper, oversampling in supervised learning is an analogous concept. During oversampling, data that is less represented in the training set is sampled at a higher rate than other data so that the model sees enough such examples from which to learn. An often-used example of such imbalanced data is cancer data, where negative responses significantly outnumber positive ones. In supervised learning, a descriptive analysis of the training data reveals data that may need to be oversampled \citep{james2013introduction}. Information shaping uses knowledge outside of the RL to adapt the distribution used to generate training scenarios and thus influence the sample paths on which the RL learns. However, whereas a descriptive analysis of the data can reveal what features need to be oversampled in supervised learning, what distributions to use for scenario generation for effective RL is unknown. In the case of this paper, we analyze a stylized version of the problem to gain insights that inform the choice of distributions. 

\section{Problem Description}\label{sec:problem}
In this section, we first provide a detailed problem description and then a formal model of the problem.

\subsection{Problem Narrative and Examples}

We consider an urban on-demand delivery service that provides delivery services over a long service horizon in a given service area. The service area is partitioned into regions. Over the course of each day on the horizon, customers from different regions request fast on-demand service, e.g., delivery from a warehouse to their homes. The service provider serves customers using the same fleet of vehicles every day. The vehicles dynamically pick up the ordered goods from the warehouse and deliver them to the customers within a predefined time span. They must deliver all loaded goods before they can return to the warehouse, and ongoing tours cannot be altered. Over the course of the day, whenever a new customer requests service, the provider determines whether the request is accepted for service and, if so, how to integrate service into the vehicles' delivery routes.

\begin{figure}[t]
\centering
\includegraphics[width=\textwidth]{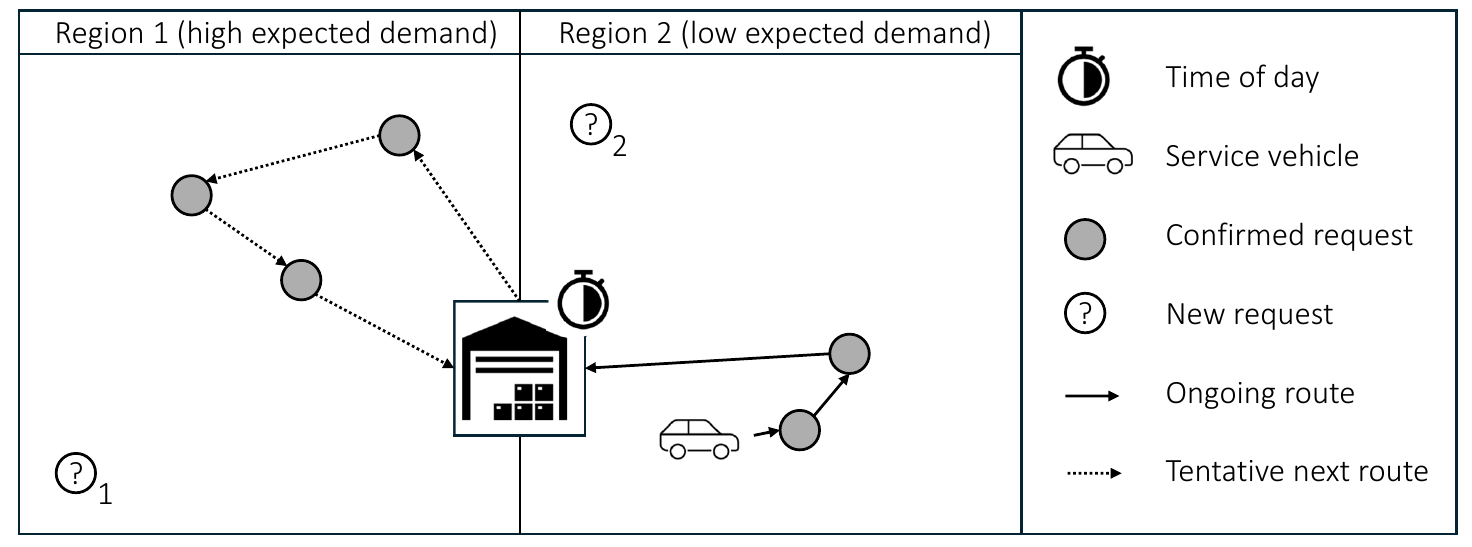}
\caption{Example for a state of the intra-day problem.}
\label{fig:intra-day}
\end{figure}

The challenges of the intra-day decision-making are illustrated in Figure~\ref{fig:intra-day}. The figure presents a small example of a state in the intra-day problem. In the example, a single vehicle performs delivery service from a central warehouse to customers requesting service in two regions, Region~1 with high expected demand and Region~2 with low expected demand. The state illustrated is in the middle of the day. The vehicle is currently serving two customers in Region~2 before returning to the warehouse to pick up the orders for three customers in Region~1. The tentative delivery route for the three orders is indicated by dotted arrows. In the state, two new customers requested service. For this example, we assume that at most one of the two orders can be served feasibly due to time constraints. The decision of which order to serve is not trivial. Order~2 could be integrated efficiently into the tentative delivery route. However, the likelihood that future orders may occur in the neighborhood of Order~1 is higher. 

Across days, the expected demand in the regions varies. Specifically, the expected demand is endogenous and varies based on the relative service availability in previous days, i.e., the percentage of customers who were offered service, from now on called \textit{service level}. In the case of a high service level, the expected demand in a region increases, and in the case of a low service level, it decreases. An exemplary demand development for the two previously considered regions is shown in Figure~\ref{fig:example_inter}. The figure shows two potential developments in realized (solid lines) and served demand (dashed lines) for the two regions over the problem's time horizon. This development is controlled by the service decisions made during each day as previously illustrated in Figure~\ref{fig:intra-day}. In the upper part of Figure~\ref{fig:example_inter}, we first assume that we pursue a policy prioritizing the high demand region, Strategy~A. Demand in Region~1 is well served leading to a high service level and an increase in demand over time, eventually approaching the maximum possible demand by the end of the horizon. For Region~2, the service level is low, and both realized and served demands diminish over the time horizon. We further show an alternative policy, Strategy~B, equally focusing on both regions in the lower part of Figure~\ref{fig:example_inter}. Here, we observe that demand in Region~1 remains rather constant while demand in Region~2 increases. In the end, both regions have about the same expected demand.


\begin{figure}[t]
\centering
\includegraphics[width=\textwidth]{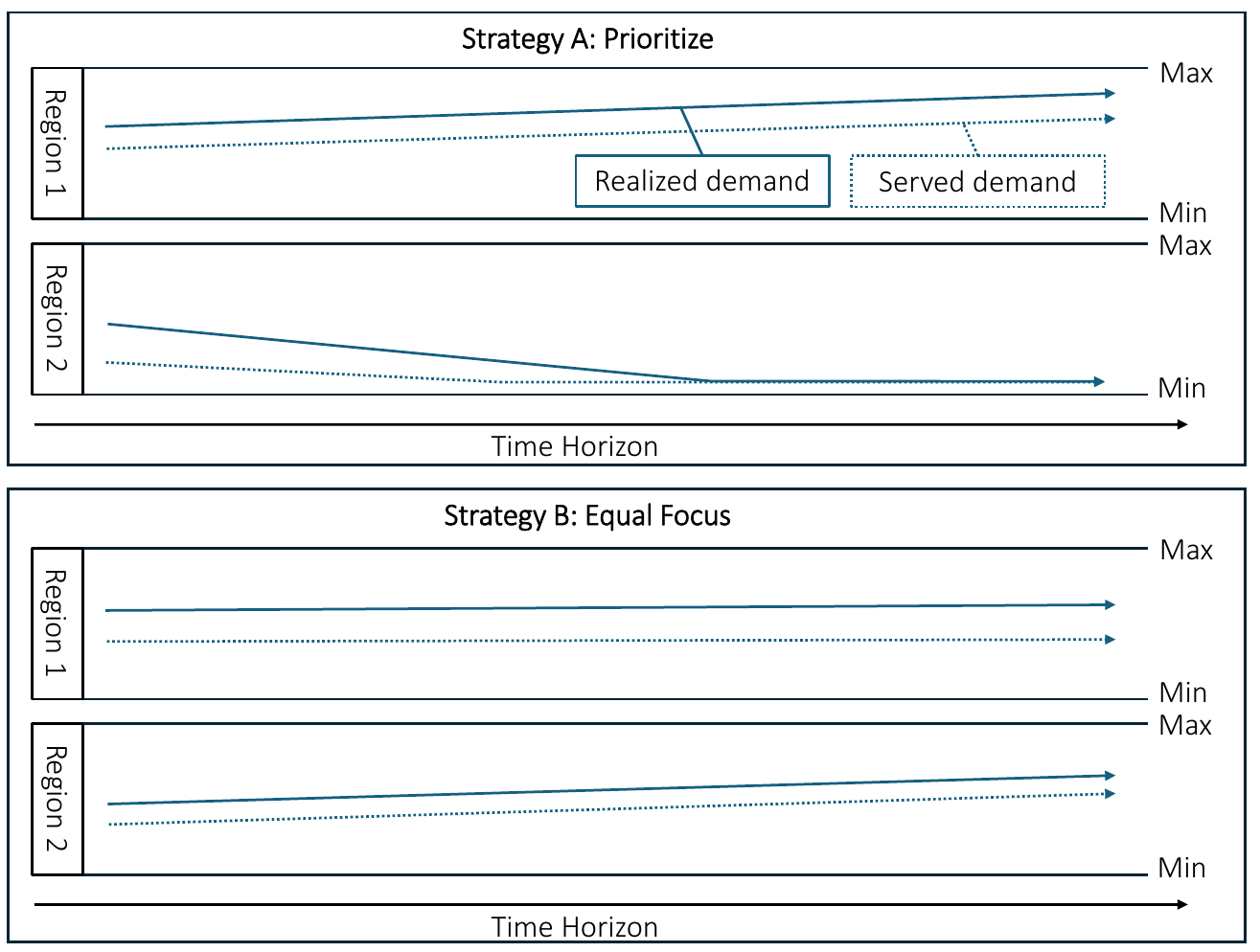}
\caption{Example for demand developments in the inter-day problem.}
\label{fig:example_inter}
\end{figure}

Figure~\ref{fig:example_inter} also relates to the objective of the service provider, to maximize the expected number of services across all regions and over the entire horizon. We note that there might be another objective: maximizing the final customer base, which is the total expected demand over all regions at the end of the time horizon. The results of our experimental study show that both goals are closely aligned. Focusing on our objective also leads to a higher customer base in the end.

\subsection{Modeling}


The problem combines two sub-problems, an intra-day problem of serving demand within the day and an inter-day problem of managing the expected demand of the regions across many days. For ease of presentation and in preparation of our methodology, we model the sub-problems individually. To differentiate the indices used in the two sub-models, we add a superscript ``$\iota$" for intra-day notations. In the following, we first sketch the intra-day model and then embed it into the inter-day model.

\subsubsection{Intra-day model.}\label{sec:intra_day}

The intra-day model is a classical on-demand delivery problem and can be modeled as a sequential decision process \citep{Powell2022RL}. For a detailed mathematical model with full notation, we refer the reader to \cite{chen2023same}. However, since in our problem, the intra-day model is connected to the inter-day demand development, we extend the classical models by adding information on observed and served demand for each sub-region in the city. In the following, we describe the individual components of the process: states, decisions, reward function, 
stochastic information and transition, and solution. We begin with the notation necessary for our solution methodology.

\paragraph{Preliminary notation.}

A service provider offers on-demand deliveries in the service area $\mathcal{Z}$. The service area is partitioned into $I$ sub-regions, $\mathcal{Z} = \cup_{i=1}^{I} \mathcal{Z}_i$. The expected demand (i.e., the number of requests) for region $\mathcal{Z}_i$ is $D_i$. The warehouse $W$ is not necessarily in the service area. Each day, customers make delivery requests during the time window $[0,T_\mathcal{C}]$. We denote the time at which the $k$th customer $c_k$ makes a request as $t(c_k)$. The service provider then needs to decide whether the service will be offered. Every accepted request must be fulfilled no later than $\bar{\delta}$ units of time since the request was made. We call $t(c_k)+\bar{\delta}$ the delivery deadline of request $c_k$. Note that the time of request and the location of the customer $c_k$ are not known until the request is made at $t(c_k)$. If the request $c_k$ resides in region $\mathcal{Z}_i$, we add $c_k$ to the regional customer set $\mathcal{C}_i$. Thus, $\mathcal{C}=\cup_{i=1}^{I} \mathcal{C}_i$ denotes the set of customers who request service in the whole service area. 

The service provider manages a fleet of $P$ vehicles, $\mathcal{V}=\{v_1,v_2,\dots,v_P\}$, to make deliveries to customers. These vehicles can make multiple trips between the warehouse and customers. They can still work after $T_\mathcal{C}$, the time at which the service provider stops receiving requests. However, due to the restriction on working hours, the vehicles must complete all assigned deliveries and return to the warehouse no later than $T_\mathcal{V}$, where $T_\mathcal{V} \geq T_\mathcal{C}$. The time it takes a vehicle to travel from one place (the warehouse or a customer) to another is determined by the function $\tau(\cdot,\cdot)$. Each vehicle spends $t_W$ units of time loading packages at the warehouse and $t_\mathcal{C}$ units of time dropping off a package at a customer. 


\paragraph{State.} A decision point occurs every time a customer $c_k$ makes a request for delivery service. Let $t_k^\iota$ denote the $k$th decision point and thus $t_k^\iota=t(c_k)$. 
At decision point $t_k^\iota$, the state $S_k^\iota$ comprises the following information: (1) Time of the decision point $t_k^\iota$, (2) location of customer $c_k$.
    (3) accepted customers that have not been loaded onto any vehicles yet with their locations and deadlines,
    (4) vehicles' currently assigned delivery tours,
    (5) region information including each region's expected demand $D_i$ and each region's number of requests accepted during $[0,t_k^\iota)$. This information on the collection of requests is needed for assessing resource allocation in the regions when considering the longer horizon.

In the initial state $S_0^\iota$ at $t=0$, there are no pre-accepted requests, and all vehicles are idle at the warehouse. 

\paragraph{Decision.} At decision point $t_k^\iota$, we need to make a decision. Let $\mathcal{X}_k^\iota$ represent the set of all feasible decisions in state $S_k^\iota$. In the intra-day model, decision $x_k^\iota=(a_k,\Theta_k)$ consists of an acceptance decision $a_k$ and a routing decision $\Theta_k$. The acceptance decision $a_k$ is defined as 
$$ a_k=\left\{
\begin{array}{rcl}
0,       &      & {\textrm{if customer $c_k$ is not offered service,}}\vspace{0.2cm}\\
1,     &      & {\textrm{if customer $c_k$ is offered service.}}
\end{array} \right. $$
Once the acceptance decision $a_k$ is made, the routing decision $\Theta_k$ updates the vehicles' planned delivery tours. As aforementioned, we do not preempt or alter ongoing tours, but can only update the next tour of a vehicle. A feasible routing decision must not violate any accepted customer's delivery deadline or any vehicle's working-hour restriction. Readers are referred to \citet{chen2023same} for a mathematical definition of the decision space. 

\paragraph{Reward.} In the intra-day problem, the reward of a decision is 1 if the customer is accepted. Otherwise, the reward is 0. Formally, the reward function $R^\iota(\cdot, \cdot)$ of a decision is defined as 
$$ R^\iota(S_k^\iota,x_k^\iota)=\left\{
\begin{array}{rcl}
1,       &      & {\textrm{if $a_k=1$,}}\vspace{0.2cm}\\
0,     &      & {\textrm{if $a_k=0$.}}
\end{array} \right. $$


\paragraph{Stochastic information and transition.} After the decision is made, the process proceeds. Let $\Omega^\iota$ be set of the exogenous information realizations, each assigned a probability. Let $\omega^\iota_{k+1}\in\Omega^\iota$ represent the realized information after $x_k^\iota$ is selected in state $S_k^\iota$. The transition function $S^{M^\iota}$ then provides a new state $S_{k+1}^\iota=S^{M^\iota}(S_k^\iota,x_k^\iota,\omega_{k+1}^\iota)$. Hereby, exactly one of the following two cases will take place. 
\begin{itemize}
    \item If $\omega_{k+1}^\iota=\emptyset$, the process terminates at the final state, and thus $S_{k+1}^\iota=S_K^\iota$.
    \item If the realized information contains a new customer request, i.e., $\omega_{k+1}^\iota=\{c_{k+1}\}$, the process transitions to a new decision point $t_{k+1}^\iota=t(c_{k+1})$. Thus, a new state $S_{k+1}^\iota$ occurs at decision point $t_{k+1}^\iota$, and the transition function updates the time, the set of customers, and the state of the vehicles.
\end{itemize}

\paragraph{Solution.} 
A solution to the intra-day problem is a policy $\pi\in\Pi$ that selects a decision for each state, where $\Pi$ is the set of all policies. For simplicity, we let $X^\pi (S_k^\iota)$ represent the decision in state $S_k^\iota$ under policy $\pi$. Then, the intra-day value of a policy $\pi$ is defined as 
$$\mathbb{E} {\bigg[\sum_{k=0}^{K} R^\iota(S_k^\iota,X^\pi(S_k^\iota))|S_0^\iota\bigg]}.$$

\subsubsection{Inter-day model.}


In this section, to distinguish the different days on the inter-day horizon, we append a subscript $\tau_m$ to the indices used in the intra-day model. For example, the expression $S_{\tau_m,k}^\iota$ represents the \textit{k}th state in day $\tau_m$. Further, inter-day notations do not have the superscript $\iota$. For example, we write $S$ to indicate inter-day states in contrast to $S^\iota$ for intra-day states, and $\mathcal{S}$ to indicate the set of all inter-day states. In the inter-day model, the horizon consists of $M$ consecutive days, i.e., day $\tau_1$, $\tau_2$, \dots, $\tau_M$. 


\paragraph{State.} There are $N$ decision points $t_{n}$ that occur every $\tau_{\textrm{upd}}$ days. An inter-day state $S_{n}$ occurs at the beginning of day $(n-1)\tau_{\textrm{upd}}+1$, where $n = 1, 2, \dots, N$. The state includes information about time and newly revealed demand distributions $\Omega_n^\iota$ for the upcoming $\tau_{\textrm{upd}}$ days, where $n$ corresponds to day $(n-1)\tau_{\textrm{upd}}+1$. Mathematically, we denote the state by $S_{n}=[(n-1)\tau_{\textrm{upd}}+1, D_{n,1},\dots,D_{n,I}]$, where $D_{n,i}$ represents the expected demand for Region~$i$ at decision point $t_n$.

\paragraph{Decision and reward.} Different from the intra-day problem, the inter-day problem involves a single crucial decision $x_{n}$: the selection of an intra-day policy to be used in the next $\tau_{\textrm{upd}}$ days. This single decision $x_{n}$ selects a policy from the intra-day policy set $\Pi$. The reward of decision $x_{n}$ is the number of services provided in the next $\tau_{\textrm{upd}}$ days, defined as 
$$R(S_{n},x_{n})=\sum_{m=(n-1)\tau_{\textrm{upd}}+1}^{n\tau_{\textrm{upd}}}\sum_{k=0}^{K} R^\iota(S_{\tau_m,k}^\iota,X^{x_{n}}(S_{\tau_m,k}^\iota))|S_{\tau_m,0}^\iota.$$
Since the customer requests are stochastic, the reward is a random variable.


\paragraph{Demand transitions.} After $\tau_{\textrm{upd}}$ days, the stochastic information $\omega_{n+1}\in\Omega$ reveals a reward, the accumulated services, and a new $\Omega_{n+1}^\iota$ for the next $\tau_{\textrm{upd}}$ days.  
Specifically, once the decision $x_{n}$ is made, the model proceeds as described in the intra-day model, providing an accumulated reward of $\tau_{\textrm{upd}}$ days and information on the service level in the individual regions. We assume that demand evolves every $\tau_{\textrm{upd}}$ days, such as every day, every week, or every month (in our experimental study, we assume monthly updates with $\tau_{\textrm{upd}}=30$), and that the demand distribution is characterized by expected demand. Thus, the intra-day exogenous information $\Omega^{\iota}$ can be represented by a set of such demand distributions. We denote the initial distribution $\mathcal{D}^1_{i}$ with expected demand $D_{1,i}$ in Region~$i$ at the beginning of the horizon. Customer demand is endogenous, and throughout the $M$-day horizon, customer demand in a given period depends on the service level that the region has received. There is no update in demand at decision point $t_1$. The first $N-1$ updates to the demand distribution occur right before decision points $t_2$, $\dots$, $t_N$. The last update takes place when the inter-day model terminates.
We let $r_{n-1,i}$ denote Region~$i$'s realized service level $\tau_{\textrm{upd}}$ days before the $n$th update. The function $D_{\text{up}}(\cdot)$ updates the demand distribution for the next $\tau_{\textrm{upd}}$ days as $\mathcal{D}_i^n=D_{\text{up}}(r_{n-1,i}), \ \ n=2,3,\dots,N+1$. Thus, the transition function $S^M$ leads to $S_{n+1} = [n\tau_{\textrm{upd}}+1, D_{n+1,1},\dots,D_{n+1,I}]$ if $n\leq N-1$, and $\Omega_{n+1}^\iota$ consists of all demand distributions whose expected value is $D_{n+1,1},\dots$, or $D_{n+1,I}$. Otherwise, the inter-day model terminates.

\paragraph{Objective and solution.} A solution to the inter-day problem is a sequence of decision rules $x = (X^{x_1}(S_{1}),\dots, X^{x_N}(S_{N}))$, where $X^{x_n}(S_{n})$ represents the intra-day policy selected for the upcoming $\tau_{\textrm{upd}}$ days. The objective function $F$ of a solution $x$ is the expected number of services over the $M$-day horizon:
\begin{equation}\label{eq:inter_day_obj}
    F(x)= \mathbb{E} {\bigg[\sum_{m=1}^{M}\sum_{k=0}^{K} R^\iota(S^\iota_{\tau_m,k},X^{x_{\lceil\tau_m / \tau_{\textrm{upd}}\rceil}}(S_{\tau_m,k}^\iota))|S_{\tau_m,0}^\iota}\bigg].
\end{equation}

The optimal inter-day policy $x^*$ is defined as    $x^*= \argmax_{x_n\in\Pi}\ F(x)$.

\section{Information Shaping for Reinforcement Learning}\label{sec:methodology}

In this section, we present our approach to solving this problem. After a short motivation and overview, we first give a general definition of information shaping as, to the best of our knowledge, the approach is new to the literature. We then present the design of the information shaping approach for our problem.


\subsection{Motivation and Overview}

The considered problem is significantly more complex compared to conventional intra-day problems. The challenges of intra-day decision-making were already highlighted in the example around Figure~\ref{fig:intra-day}. In addition to solving the problem for a specific day, decisions also need to anticipate the demand development over future days (together with future decision-making). As illustrated in Figure~\ref{fig:example_inter}, two different intra-day decision-making strategies lead to very different demand developments. Considering such developments is crucial since optimizing individual days usually leads to a high service level in high demand regions (Region~1 in the example) but poor service in less developed regions (Region~2 in the example) \citep{ulmer2018value}.


In the literature, one prominent way to provide effective decision-making for the intra-day problem is reinforcement learning. Although the existing literature has shown promising results when applying RL to on-demand delivery problems, the literature does not consider long-term demand developments. Furthermore, the literature has shown that for such a problem, RL works well only for a rather ``overseeable" horizon of a few days. As the training horizon is extended, the reward of a decision offers very little information about the quality of the decision \citep{ulmer2018value}. Consequently, preliminary tests of running RL for the full problem did not work. Thus, we introduce a new concept that allows training on the intra-day problem only while considering demand evolution: RL with information shaping. Information shaping is a mechanism for controlling learning through the generation of scenarios for training in reinforcement learning. The goal of information shaping is to guide RL toward effective policies. In our case, the stochastic information of the inter-day process is the demand development. Thus, we analyze the steady state of a simplified version of the inter-day problem to derive properties of optimal expected demand values. We then use these properties to shape the training data for RL policies that are applied only to the intra-day problem. In the context of the problem discussed in this paper, information shaping is therefore used to connect the intra- and inter-day problems.
In the following, we formally introduce information shaping for reinforcement learning and then demonstrate how we apply it to the problem in this paper. 


\subsection{Conceptualizing Information Shaping}\label{sec:shaping}


In general, a sequential decision problem or Markov decision process can be characterized by the set $\{\mathcal{S}, \mathcal{X}, R, \Omega, S^M, F\}$, where $\mathcal{S}$ is the set of all states, $\mathcal{X}$ the decision space, $R$ the reward function that is usually dependent on a particular state and decision, $\Omega$ the exogenous information, $S^M$ the transition function that describes the transition from one state to the next given the current state $S$, decision $x$, and newly realized exogenous information $\omega$. The function $F$ is the objective of the problem.

In theory, we can find the optimal solution to $\{\mathcal{S}, \mathcal{X}, R, \Omega, S^M, F\}$ by using backward induction to solve the Bellman equation. The Bellman equation is given by
\begin{equation}\label{eq:bellman}
V(S) = \max_{x \in \mathcal{X}}\{R(S, x) + \mathbb{E}[V(S^M(S,x))]\}.
\end{equation}
However, many sequential decision problems encounter the ``curses of dimensionality" in the state and decision spaces. In this paper, the intra-day problem alone has such a significant state space that it is impossible to solve the problem exactly. Because of its ability to learn and interact with dynamic environments, RL has gained significant attention in the solution of sequential decision problems, including dynamic routing problems \citep{hildebrandt2023opportunities}. RL offers the advantage that it enables offline policy learning, eliminating the need for real-time optimization, and thus facilitating real-time decision-making. Further, unlike many heuristic approaches in the literature, RL is capable of accounting for both the immediate and future rewards of a Bellman equation. 

RL requires a set of training episodes, with responses that guide the learning. Such guidance could follow from expert opinion, which is the case in HLRL. Alternatively, in training robots, scenarios may come from learning-by-doing as a team at Google's DeepMind did in training robots to play soccer \citep{haarnoja2024learning}. In other cases, however, such real-time learning may be dangerous such as in medication dosing \citep{zadeh2023optimizing} or too expensive, as is the case for the problem in this paper. In these circumstances, the scenarios are generated via simulation. For a general sequential decision problem $\{\mathcal{S}, \mathcal{X}, R, \Omega, S^M, F\}$, these scenarios are simulated by sampling the exogenous information space $\Omega$, generating sample realization $\hat{\omega}$, and coupling that with the transition function $S^M$ to move from a state $S$ to a new state $S^{\prime}$. Information shaping focuses on this case. 



In information shaping, we replace the given information space $\Omega$ with a modified space $\tilde{\Omega}$. We can design $\tilde{\Omega}$ to favor features of states or sample paths that are more likely under ``good" policies, similar to what takes place with oversampling in supervised learning as discussed in Section~\ref{lit:rl}. In this paper, we achieve this effect through control of the demand distributions on which we train. We determine what these distributions should be by characterizing the optimal solution to a simplified version of the problem.

\subsection{Information Shaping for Same-Day Delivery with Endogenous Demand}


The key question in implementing information shaping is how do we design $\tilde{\Omega}$. For the problem studied in this paper, we seek a policy that will guide demand to ``good" demand distributions across the regions while also performing well in that steady state. To achieve this goal, we need to know what the ``good" demand distributions are. We do that by optimizing a steady-state version of the problem. This approach abstracts away from the complexities of demand transition dynamics. In the remainder of this section, we first propose a stylized model of the problem. Then, we characterize the optimal solution of the stylized model in the convergence case and show how we use the insights to shape the information used to train an RL-algorithm.



\subsubsection{Description and analysis of the stylized problem.} \label{sec:shaping_problem}


To make rigorous analysis tractable, we reduce the inter-day problem to a simple control problem where decisions are made about the amount of resources to be allocated to different delivery regions. For simplicity, we consider the case of two regions. Consider two classes of customers $\mathcal{C}_1$ and $\mathcal{C}_2$ such that these classes of customers request service at rates $\lambda_1$ per day and $\lambda_2$ per day, respectively.  We need to determine the rates $r_1, r_2 \in [0,1]$ in the state at which we serve classes $\mathcal{C}_1$ and $\mathcal{C}_2$, respectively. The objective is to maximize $\lambda_1 r_1 +\lambda_2 r_2$ subject to several constraints, especially in the available resources. The problem is complicated by the fact that demand rates $\lambda_1$ and $\lambda_2$ are endogenous and therefore vary positively with $r_1$ and $r_2$, respectively.  We also assume demand saturation, i.e., that the rate of change of $\lambda_1(r)$ and $\lambda_2(r)$ decreases in $r$ and may even be bounded (e.g., there exists an $r$ such that $\lambda(r) = \lambda(r^{\prime})$ for every $r^{\prime} > r$). Further, due to the consolidation effect, the marginal cost of serving a customer decreases as the number of customers served increases, and thus relative resource consumption varies negatively with $r_1 \times \lambda_1$ and $r_2 \times\lambda_2$, respectively. 

To account for this consolidation effect, we assume that the expected time required to serve customers in each region is governed by the well-known Beardwood, Halton, and Hammersley approximation for the TSP \citep{beardwood1959shortest}. This approximation is generally given by $\beta \sqrt{nA}$, where $\beta>0$ is a constant, $n$ is the number of customers being served, and $A$ is the area of the geography being serviced. For our purposes, we assume that $\beta$ and $A$ are given and the same for customer classes $\mathcal{C}_1$ and $\mathcal{C}_2$. Thus, the expected time needed to serve customers in $\mathcal{C}_1$ is $\beta \sqrt{r_1 \lambda_1 A}$ and in $\mathcal{C}_2$ is $\beta \sqrt{r_2 \lambda_2 A}$. 
We are constrained by the total delivery resources $T$, i.e., 
\begin{equation}
   \beta \sqrt{r_1 \lambda_1 A} +  \beta \sqrt{r_2 \lambda_2 A} \leq T.
\end{equation}

To model the relationship between demand and service level for each region, we use a natural log function. This function satisfies that demand increases as the service rate increases and that its second derivative is negative. Specifically, the demand function is in the form 
\begin{equation}
    \lambda_i=log(M_i r_i+1),\ \ \ \ i=1,2.
\end{equation}
The positive constant $M_i$ scales the log curve horizontally over the interval $r_i\in [0,1]$. To avoid negative demand for any service rate between $0$ and $1$, we add $1$ to $M_i r_i$. 

Then, the steady-state problem can be modeled as
\begin{align*}
    \max \ \ & \ \lambda_1 r_1 +\lambda_2 r_2 \\
    \textrm{s.t.} \ \ & \beta \sqrt{r_1 \lambda_1 A} +  \beta \sqrt{r_2 \lambda_2 A} \leq T \\
    & \lambda_i = log(M_ir_i+1)\ \ \ \ \ \ \ i=1,2\\
    & r_1, r_2 \in [0,1].
\end{align*}


A trivial case is when the service provider has enough delivery resources to serve all customers. In this case, the optimal solution is simply to accept all requests, i.e., $r_1=r_2=1$. As mentioned earlier, it is not uncommon in practice that giant service providers like Amazon may accept any on-demand delivery requests. 
In this paper, our analysis focuses on the non-trivial case when the service provider cannot feasibly serve all customers. Nevertheless, our findings will also benefit service providers with adequate resources in managing demand and improving the efficiency of resource allocation in different neighborhoods. In this section, we present the main findings. We refer the reader to Appendix~\ref{appendix:theorem1} for details.

For the non-trivial case, we assume the binding condition of the resource constraint in Proposition~\ref{binding_assumption}.
\begin{proposition}\label{binding_assumption}
    In the non-trivial case, there exists an optimal solution $r_1^*$ and $r_2^*$ where the resource constraint is binding, i.e., $\beta \sqrt{r_1 \lambda_1 A} +  \beta \sqrt{r_2 \lambda_2 A} = T$. 
\end{proposition}

\begin{proof}[Proof of Proposition \ref{binding_assumption}.] We prove the result using contradiction. Assume that, in the non-trivial case, the optimal solution $r_1^*$ and $r_2^*$ does not consume all delivery resources, i.e., $\beta \sqrt{r_1^* \lambda_1 A} + \beta \sqrt{r_2^* \lambda_2 A} < T$. The corresponding objective value is $\lambda_1 r_1^* +\lambda_2 r_2^*$. Then, consider $\delta >0$ such that $\beta \sqrt{(r_1^*+\delta) \lambda_1 A} + \beta \sqrt{r_2^* \lambda_2 A} < T$ still holds. Note that $\delta$ should also satisfy $r_1^*+\delta\leq 1$. (In the case $r_1^*$ itself is already $1$, we can swap the region indices $1$ and $2$ in the remainder of this proof. The non-trivial case requires that $r_1^*$ and $r_2^*$ are not both $1$.) For this new solution $r_1^*+\delta$ and $r_2^*$, the objective value is $\lambda_1 (r_1^* +\delta)+\lambda_2 r_2^*$. Calculating the difference in objective value, we obtain $$\lambda_1 (r_1^* +\delta)+\lambda_2 r_2^*-(\lambda_1 r_1^* +\lambda_2 r_2^*)=\lambda_1\delta >0.$$ 

Since the new solution $r_1^*+\delta$ and $r_2^*$ has higher objective value than solution $r_1^*$ and $r_2^*$, this contradicts that the solution $r_1^*$ and $r_2^*$ is optimal. Hence, by contradiction, the optimal solution exists when the resource constraint is binding.

\end{proof}

Let the objective be a function denoted by $f(r_1,\lambda_2,r_2,\lambda_2)=\lambda_1 r_1 +\lambda_2 r_2$. With Proposition~\ref{binding_assumption}, we can solve the binding condition of the resource constraint for $r_1\lambda_1$ and substitute it in the objective function. Further, we represent $\lambda_2$ in terms of $r_2$ according to the demand model. Thus, the objective function becomes a function in only $r_2$,
    \begin{equation}\label{main_substitution}
        f(r_2)=\frac{T^2-2T\beta \sqrt{A log(M_2 r_2+1)r_2}+2A\beta^2log(M_2r_2+1)r_2}{A\beta^2},\ 0\leq r_2\leq 1.
    \end{equation}
We then differentiate the objective function with respect to $r_2$ and obtain the derivative
    \begin{equation}\label{obj_derivative_main_body}
        f'(r_2)=\frac{[log(M_2r_2+1)+M_2r_2+log(M_2r_2+1)M_2r_2]\cdot[2\beta\sqrt{A log(M_2r_2+1)r_2}-T]}
        {\beta \sqrt{Alog(M_2r_2+1)r_2}(M_2r_2+1)}.
    \end{equation}
Note that $f(r_2)$ is defined at $r_2=0$ but $f'(r_2)$ is not. Extrema can occur at where $f'(r_2)$ equals zero or where $f(r_2)$ is defined but $f'(r_2)$ is not defined. Therefore, we consider two subsets of $0\leq r_2\leq 1$, i.e., $0< r_2\leq 1$ and $r_2=0$. The analyses lead to Theorem~\ref{theorem:optimal_main_body}. 
    \begin{theorem}\label{theorem:optimal_main_body}
    In the steady state, it can be optimal to invest all delivery resources in one region or equally allocate resources to the two regions. 
    \end{theorem}

Figure~\ref{fig:shaping} illustrates the general idea of Theorem~\ref{theorem:optimal_main_body} for the case of two equal regions in a stylized manner. There are two options: either the maximum number of services can be reached by setting $r_1=1$ (or $r_1=0$), or, when setting $r_1=r_2=0.5$. To prove Theorem~\ref{theorem:optimal_main_body}, we first prove a lemma for the interval $0< r_2\leq 1$.

    \begin{lemma}\label{lemma:main_excluding_0}
    Assume the service level $r_2$ is in the interval $(0,1]$. Then, in the steady state, it can be optimal to allocate half of the delivery resources to each region.
    \end{lemma}

    \begin{proof}{Proof of Lemma \ref{lemma:main_excluding_0}.}
    
    The denominator in Equation~\ref{obj_derivative_main_body} is nonnegative for all $0<r_2\leq 1$. Setting the numerator equal to $0$ to find any critical points, we get 
    \begin{equation}\label{main_first_critical}
        log(M_2r_2+1)+M_2r_2+log(M_2r_2+1)M_2r_2=0.
    \end{equation}
    \begin{equation}\label{main_second_critical}
        \textrm{or\ \ \ }\ 2\beta\sqrt{A log(M_2r_2+1)r_2}-T=0.
    \end{equation}
    Equation~\ref{main_first_critical} has no solution because all three terms on the left-hand side are strictly greater than $0$ when $0<r_2\leq 1$. Rearranging the terms in Equation~\ref{main_second_critical} gives us
    \begin{equation}\label{main_half_resource}
        \beta\sqrt{A log(M_2r_2+1)r_2}=\frac{T}{2}.
    \end{equation}
    Because of the demand equation $\lambda_2=log(M_2r_2+1)$, Equation~\ref{main_half_resource} becomes
    \begin{equation}\label{main_half_resource_2}
        \beta\sqrt{A\lambda_2 r_2}=\frac{T}{2}.
    \end{equation}    
    The left-hand side is the delivery resources the service provider invests in region $\mathcal{Z}_2$. The right-hand side is exactly half of the total delivery resources available to the service provider. Hence, a critical point occurs when half of delivery resources are allocated in each region.
    \end{proof}

 
    Lemma~\ref{lemma:main_excluding_0} identifies the situation when a critical point $r\textsubscript{critical}$ occurs in the interval $0<r_2\leq 1$. Then, we include $r_2=0$ in our analysis. When $r_2=0$, the service provider invests all delivery resources $T$ in the other region $\mathcal{Z}_1$, not necessarily accepting all requests. (Note that the TSP approximation $\beta\sqrt{A\lambda_1 r_1}$ represents the delivery resources invested, e.g., time, while $\lambda_1 r_1$ represents the actual number of services offered.) At $r_2=0$, the objective value is $f(0)=\frac{T^2}{A\beta^2}$, which can be maximal if $f(0)>f(r\textsubscript{critical})$. Because the exact expression of $r\textsubscript{critical}$ is lacking, we can only conclude that it is possible that the optimal solution is $r_2=0$. Finally, the above analyses complete the proof for Theorem~\ref{theorem:optimal_main_body}.

Theorem~\ref{theorem:optimal_main_body} helps us identify a candidate set of optimal solutions. While we acknowledge that we cannot determine all possible solutions and may not know which ones are truly optimal, we can identify some potential values. Further, Theorem~\ref{theorem:optimal_main_body} provides a structured approach to narrowing down these candidates, allowing us to focus on the most promising ones in the solution set. 


If we assume that it is optimal to equally allocate delivery resources, Corollary~\ref{optimal_ratio_theorem_main_body} identifies that the ratio of service levels should vary inversely with the ratio of demands. That is, the relatively more requests in a region, the relatively lower the service level. The proof is presented in~\ref{sec:corollary_proof}.

\begin{corollary}\label{optimal_ratio_theorem_main_body}
    When it is optimal to equally allocate delivery resources in the steady state, the optimal decisions $r^*_1$ and $r^*_2$ satisfy $\frac{r^*_1}{r^*_2} = \frac{\lambda_2}{\lambda_1}$.
\end{corollary}





\subsubsection{From theory to practice.}\label{sec:information_shaping}

The analytical insights from the previous section reveal that either resources should be distributed equally across the regions or some regions should be prioritized over others. In this section, we explain how we incorporate these insights into the training process of our RL method.

For each of the two cases, we create a RL policy. The algorithmic procedures of both policies are the same but they differ in how we create the training data via information shaping. In our case, the training data comprises daily demand realizations when training the RL policies. This is captured by the expected demand of each region and then the resulting distribution, e.g., a Poisson process to generate requests according to the expected demand value. Until now, RL is trained with one fixed expected demand value. To allow for capturing demand evolution and guiding our policy toward desired (inter-day) states, in our case, we define the training data by a distribution of expected demand values for each region. Each expected value defines a different demand distribution. We use distributions of expected values to cover the demand evolution. We further add the expected demand values for each region as features to the RL policy.

Mathematically, let $\mathbb{E}\widehat{\mathcal{D}}_i$ be the expected demand value for Region~$i$. Then, we create the expected demand distribution as $\mathcal{N}(\mathbb{E}\widehat{\mathcal{D}}_i,0.5\mathbb{E}\widehat{\mathcal{D}}_i)$, i.e., normally distributed around the mean with a coefficient of variation of $0.5$ to cover a broader range of demand distributions. Now, the question arises about how to set the \textit{expected} expected demand values $\mathbb{E}\widehat{\mathcal{D}}_i$ for each region. 

For the policy that distributes resources equally, denoted IRL-E, we create the same expected demand distribution for all regions. The \textit{expected} expected demand is set to the average initial demand of all regions:

\begin{equation}
    \mathbb{E}\widehat{\mathcal{D}}_i= |I|^{-1 }\sum\limits_{i \in I} \mathcal{D}^0_i.
\end{equation}
The outcome is illustrated by the blue curve in Figure~\ref{fig:shaping_d} for the two-region case. Both regions have the same expected expected demand and the same distribution of expected demand values. For the policy that prioritizes regions, denoted IRL-P, in theory, we would set demand for the prioritized regions high and for the non-prioritized regions to zero. However, to model demand evolution, we create a slightly softer ``80-20" distribution where we set the expected expected demand for priority regions to be four times higher than the demand for non-priority regions. The overall expected demand to distribute is again the sum of initial expected demand. Regions to be prioritized are either regions with higher initial demand than average or regions that are closer to the warehouse. The procedure is shown via the orange-colored curves in Figure~\ref{fig:shaping_d} for the two-region case. The distribution for non-priority Region~2 is closer to the minimum expected demand value while the distribution for priority Region~1 is closer to a potentially maximum expected demand value.
We further see that the variance for the latter is higher to allow evolution toward the high demand values.



\begin{figure}[t]
    \centering
	\begin{minipage}[b]{0.485\linewidth}
		\centering
		\includegraphics[width=\linewidth]{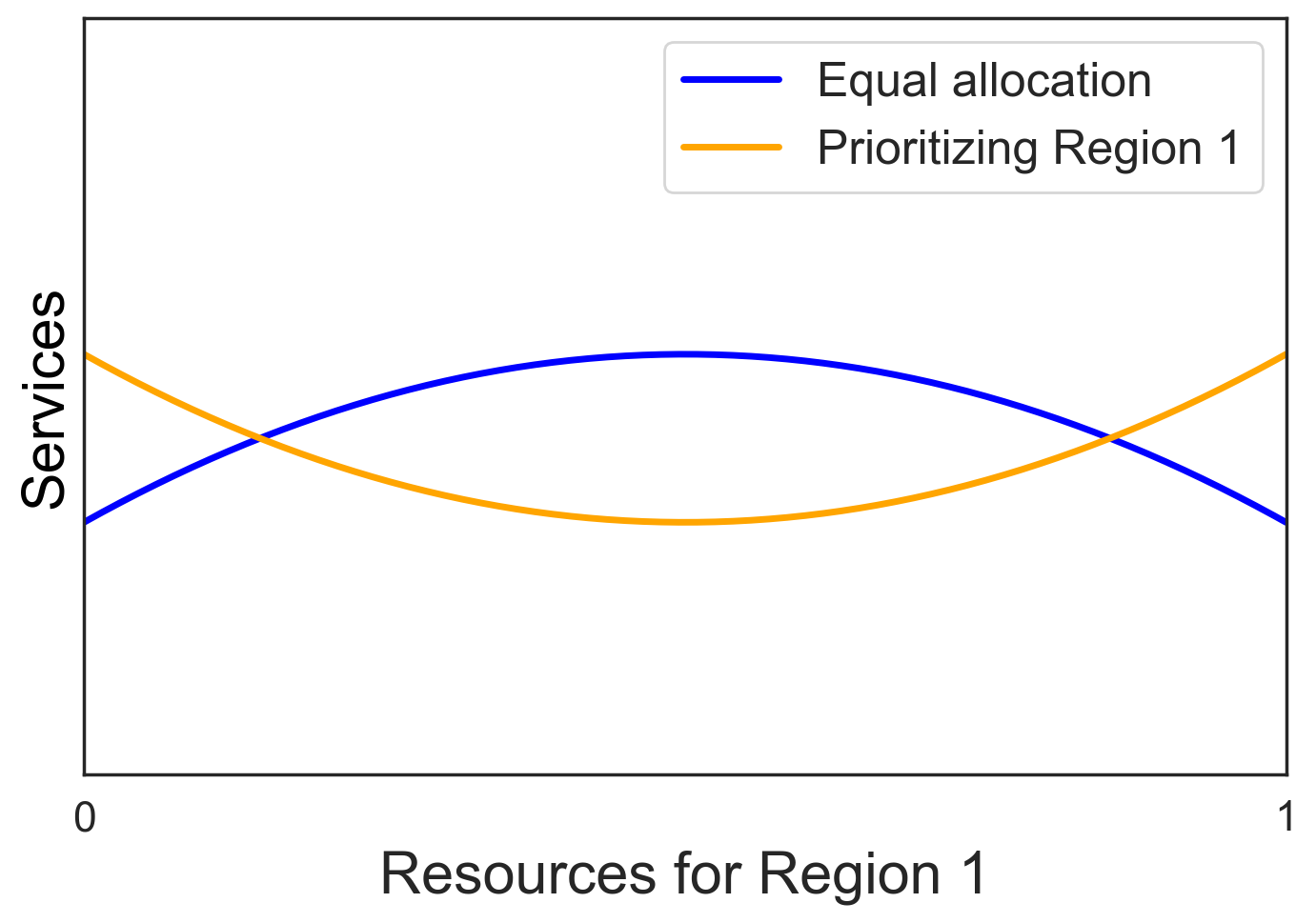}
		\subcaption{Objective value vs. resources invested in Region 1.}
            \label{fig:shaping}
        \end{minipage}
        \begin{minipage}[b]{0.495\linewidth}
		\centering
        \includegraphics[width=\linewidth]{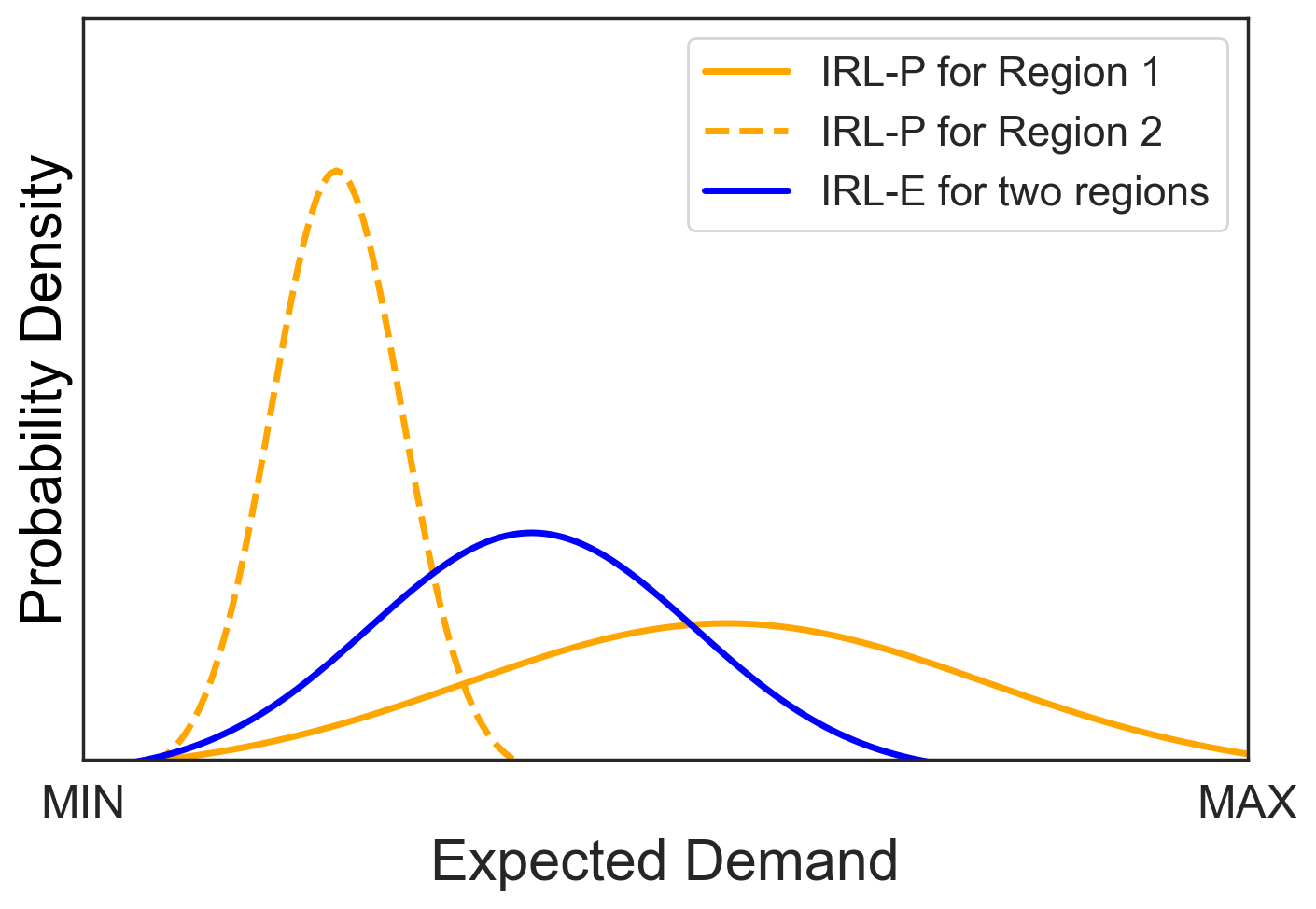}
		\subcaption{Demand distribution vs. different strategies.}	
    \label{fig:shaping_d}
        \end{minipage}
        
	\caption{Illustration of information shaping applied to two different strategies.}
        
\end{figure}

Finally, it is worth noting that our method does not simply train with different distributions. Rather, we shape the training information to connect the intra- and inter-day problems, resulting in only one policy. This policy is sensitive to the demand distribution of the inter-day state. For implementation details, we refer the reader to~\ref{sec:additional_details}.

\section{Experiments and Implications}\label{sec:experiments}
In this section, we present our experiments and analyze key dimensions of the problem.


\subsection{Experimental Design}

We first describe the basic settings, demand evolution models, and benchmark policies.

\subsubsection{Basic settings.}

\paragraph{Intra-day problem.} The service provider receives customer requests for $T=7$ hours. The arrival of requests follows a Poisson process. If a request is accepted for service, it must be served within $\bar{\delta}=4$ hours. The provider employs a fleet of $P=5$ vehicles. These vehicles operate for $T_\mathcal{V}=8$ hours, one additional hour after the service provider stops receiving requests. Their travel speed is $30$km/h. To account for road distances and traffic, we convert Euclidean distances into travel times using the approach in \citet{boscoe}. To handle the real-time routing, we implement an insertion heuristic, similar to those found in \cite{azi,chen2023same}. Vehicles are assumed to be uncapacitated because of mostly small delivery items and limited stops per route. The loading time at the warehouse and service time at a customer are $t_W=3$ minutes and $t_\mathcal{C}=3$ minutes, respectively.


\paragraph{Geography.} We design three geographies of the service area, shown in Figure~\ref{fig:geography}. Two geographies are symmetrical with respect to regions and the warehouse and differ in the initial demand per region, and one geography exhibits the same initial demand but different distances to the warehouse. These geographies are defined as follows:
\begin{figure}[h]
    \centering
	\begin{minipage}[b]{0.32\linewidth}
		\centering
		\includegraphics[width=\linewidth]{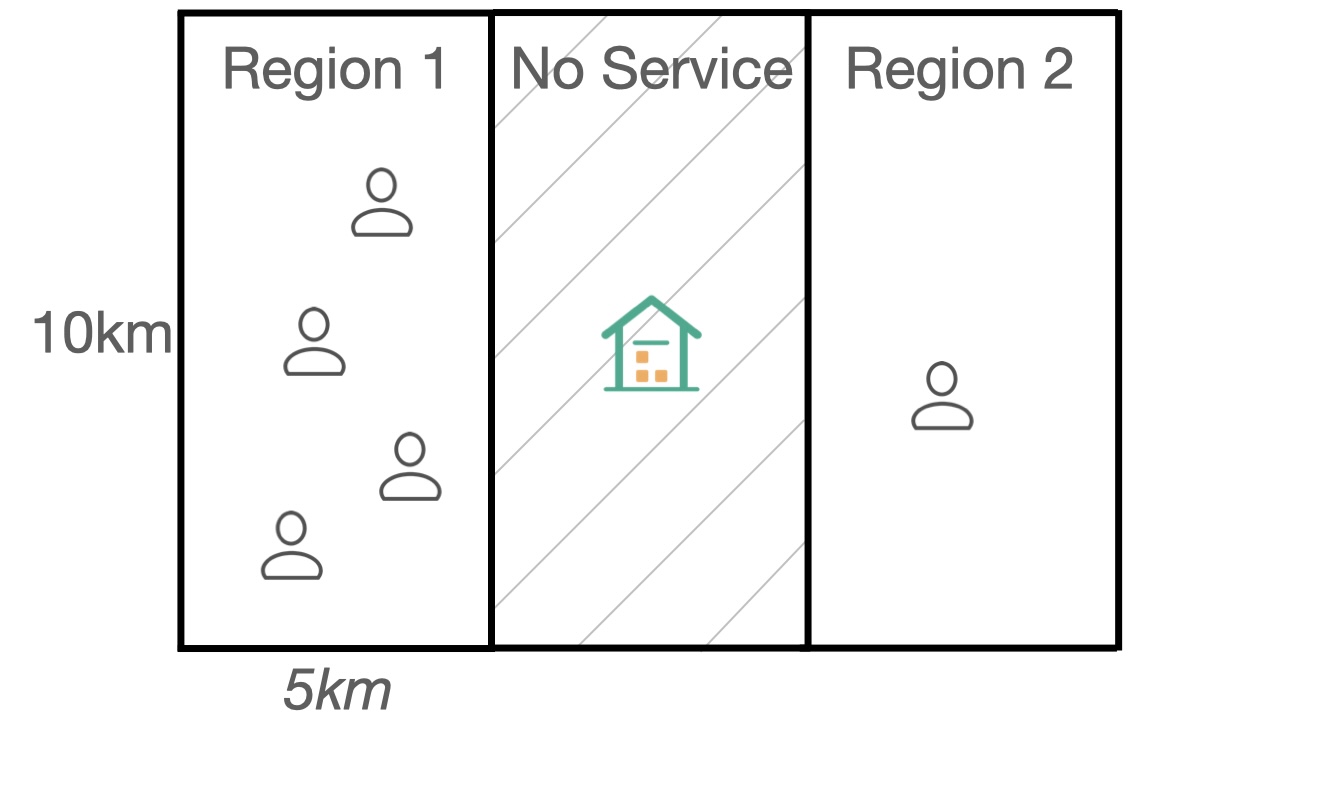}
		\subcaption{}
        \end{minipage}
	\begin{minipage}[b]{0.32\linewidth}
		\centering
		\includegraphics[width=\linewidth]{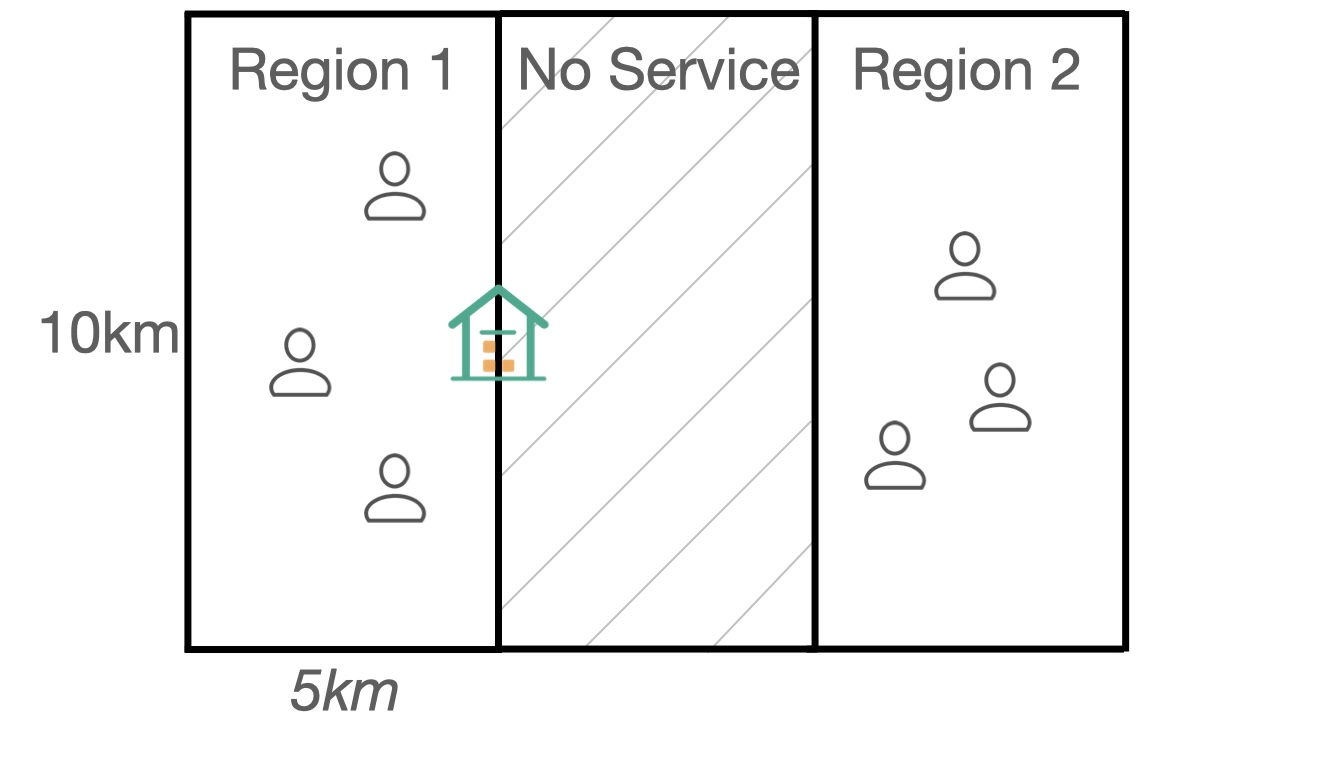}
		\subcaption{}	
        \end{minipage}
        \begin{minipage}[b]{0.31\linewidth}
		\centering
		\includegraphics[width=\linewidth]{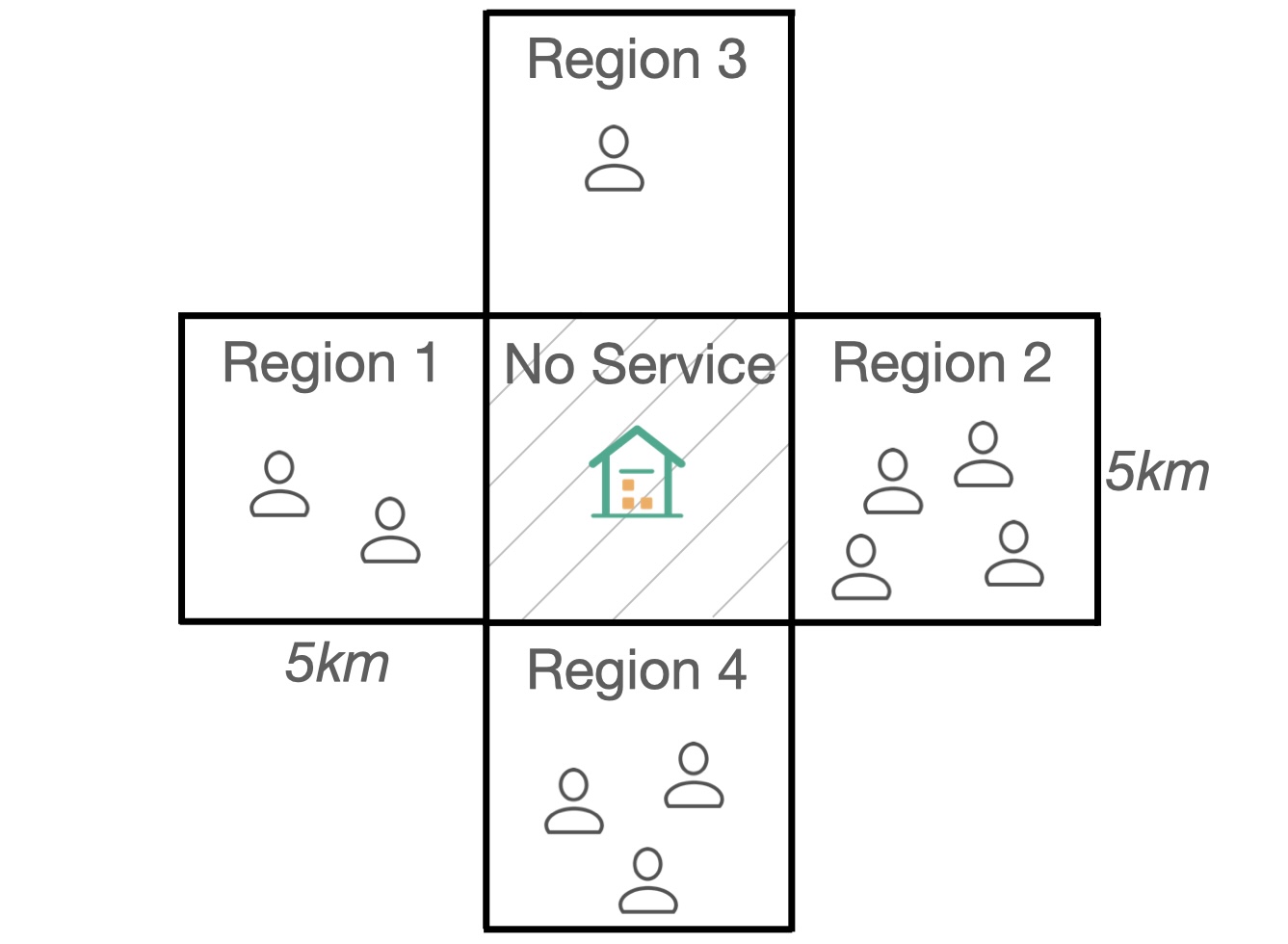}
		\subcaption{}	
        \end{minipage}
	\caption{Geography of the service area.}
	\label{fig:geography}
\end{figure}
 

\begin{enumerate}[label=(\alph*),leftmargin=*]
    \item  In this geography, regions differ in the initial demand for services. There are two regions, Region~1 and Region~2. We create a \textit{No Service} area between the two regions, e.g., a business district. The warehouse is located in the center of the \textit{No Service} area. On day~1, the two regions begin with $D_{1,1}=200$ and $D_{1,2}=50$ expected requests per day, respectively. The $x$- and $y$-coordinates of customers are first generated from independent and identical normal distributions with a mean of $5$km and a standard deviation of $3$km. Then, we adjust the $x$-coordinate to reflect the \textit{No Service} area. This is done by adding $5$km to the $x$-coordinate of customers in Region~2. Note that due to the standard deviation of $3$km, there may be customers residing slightly outside the boxes.  
    \item We consider a setting with equal initial demand but different distances to the warehouse. The warehouse is located on the right border of Region~1. The coordinates of customers are generated using the same method as in~(a). Different from the previous geography, both regions in~(b) have the same initial demand $D_{1,1}=D_{1,2}=125$.    
    \item This geography is similar to~(a), where regions differ in initial demand. In this geography, we consider four regions that initially start with $D_{1,1}=50$, $D_{1,2}=100$, $D_{1,3}=25$, and $D_{1,4}=75$ expected requests per day, respectively. Because the size of each region is smaller than that in~(a), we use a uniform distribution over $0$km to $5$km when generating the $x$- and $y$-coordinates of customers.
\end{enumerate}


\subsubsection{Demand models.}\label{sec:training_and_evalluation}

We consider a two-year horizon ($M=720$ days) and update each region's expected demand every $\tau_{\textrm{upd}}=30$ days. Then, the $24$ demand updates occur at the end of days~30, 60, \dots, 720. At each update, to obtain the realized service level for a region, we divide the total number of services provided in that region by the total number of requests received from that region in the last 30 days (see~\ref{sec:additional_details} for details). 

For demand evolution, we consider two models, a capacitated model and an uncapacitated model. 

\paragraph{Capacitated model.} This model assumes a maximum demand level $D_{\text{cap},i}$ in Region~$i$, expressed as 
    $D_{n,i}=(1-\alpha)\cdot D_{n-1,i}+\alpha\cdot D_{\text{cap},i}\cdot r_{n-1,i}$. Note that $\mathcal{D}_{i}^n$ represents a distribution of demand, which is characterized by the expected value $D_{n,i}$. This model is similar to exponential smoothing forecast models. The parameter $\alpha\in(0,1)$ indicates the customer ``sensitivity", i.e., how much weight is put on the latest realized demand. A higher value of $\alpha$ puts more weight on the last month. Since both $\alpha$ and $r_{n-1,i}$ cannot exceed values of one, the maximum demand is capped by $D_{\text{cap},i}$.

\paragraph{Uncapacitated model.} This model operates with a service-level threshold, denoted as $\bar{r}$. In this model, the expected demand level for a region is updated as   $D_{n,i}=D_{n-1,i} + D_{n-1,i}\cdot (r_{n-1,i}-\bar{r})$. The use of threshold $\bar{r}$ is similar to the \textit{reference point} or \textit{customer expectation} in the literature \citep{mathies2011role}. A service level above the threshold attracts more customers while a value below results in churn. We note that while this demand model is theoretically uncapacitated, growing demand infinitely is practically impossible because the resources to satisfy demand are limited. More demand likely leads to lower service levels and eventually a convergence toward $\bar{r}$.


The two demand models differ primarily in two aspects. First, the capacitated model can only approach a certain demand level $D_{\text{cap},i}$, whereas the uncapacitated model has the potential to theoretically increase infinitely. Second, the capacitated model incorporates long-term memory, as it calculates the new expected demand using all previous demand levels and historical service levels with varying weights. Conversely, the uncapacitated model has short-term memory, which considers the information only since the last update. We use the following values for the evaluation parameters, $\alpha\in\{0.25, 0.5, 0.75\}$ and $\bar{r}\in\{0.5, 0.55, 0.6, 0.65, 0.7, 0.75, 0.8, 0.85\}$. We use a value of $D_{\text{cap},i}=250$ for each of the two regions in Geographies (a) and (b), and $D_{\text{cap},i}=125$ for each of the four regions in Geography~(c). 

In total, we obtain 33 different experimental settings, 11 for each geography. For each setting, we evaluate a policy 100 times with different demand realizations in the first month. (Demand realizations in later months are naturally different due to endogenous demand models.)

\subsection{Benchmark Policies}

In addition to IRL-E and IRL-P, we implement five benchmark policies (see~\ref{appendix:benchmark} for details):


\begin{itemize}[leftmargin=*]
    \item \textit{Myopic policy.} The myopic policy accepts a customer request whenever it is feasible to serve the customer. When there are multiple vehicles that can feasibly make the delivery, the service provider assigns the delivery to the vehicle with the smallest increase in travel time.
    \item \textit{Intra-day policy.} The intra-day policy is similar to an existing method from the literature \citet{deepQ}. Without consideration of demand evolution, the intra-day policy maximizes the number of services provided merely on a daily basis. In all training instances, the expected demands are constant at 200 and 50 for two regions and 125 for all four regions. 
    \item \textit{Bucket policy.} This policy can be seen as action shaping. Motivated by Theorem~\ref{theorem:optimal_main_body}, the bucket policy seeks to balance the delivery resources allocated to regions. In a day $\tau_m$, the service provider accepts at most $\frac{D_{\tau_m,1}+\dots+D_{\tau_m,I}}{I}$ requests in each of the $I$ regions. Unless a region reaches this limit, the service provider accepts every request that can be feasibly fulfilled. This policy does not need training either. 
    \item \textit{Reward-shaping policy (RRL).} In the literature, reward shaping is designed to incorporate human domain knowledge into the learning process to speed up the convergence rate by modifying the reward used during training \citep{ng1999policy}. This policy differs from the intra-day policy in that, during training, we assign customers different values of reward. If the maximum expected demand across all regions is $D_{\max}$, then, motivated by Corollary~\ref{optimal_ratio_theorem_main_body}, the reward for a customer from Region~$i$ is $\frac{D_{\max}}{D_{1,i}}$. 
    \item \textit{Manipulated-reward policy (MRL).} This policy serves as a hybrid of the intra-day and reward-shaping policies. We train the policy the same way we train the intra-day policy. So, every customer has a reward of~1 during training. When implementing the trained policy, we artificially manipulate the reward the same way we do with the reward-shaping policy. In a day $\tau_m$, if expected demands for the regions are $D_{\tau_m,1}, \dots, D_{\tau_m,I}$, respectively, then the reward of a customer from Region~$i$ is $\frac{D_{\max}}{D_{\tau_m,i}}$. 
\end{itemize}

\subsection{Objective Value}

We first compare the objective values of the policies. For a better comparison, we use the intra-day policy as the base benchmark. For each instance setting, we calculate the average number of customers served per day over the two-year horizon. We then compute the improvement of a policy over the intra-day policy as 
    $\frac{\text{Policy}-\text{Intra-Day Policy}}{\text{Intra-Day Policy}}$. In addition, we also calculate the average improvement if the best IRL policy is selected in a particular instance setting, denoted by IRL (Best). Finally, we calculate the average improvement over all 33 instance settings. Figure~\ref{fig:services_value_of_shaping_by_policy} presents the results by dark (blue-colored) bars, sorted by the value of improvements. We observe that the intra-day policy outperforms all other benchmark policies, even the ones that anticipate demand evolution via action shaping (the bucket policy) or reward shaping (RRL, MRL). Thus, optimizing each day individually is a valid option compared to alternative strategies:
\begin{insight}\label{ins:intra_day}
    High performance intra-day optimization is not necessarily ineffective even in the presence of inter-day developments because it leads to higher service levels and can grow demand implicitly. 
\end{insight}

However, our proposed policies improve upon the intra-day policy by $7.1\%$ for IRL-E, $19.1\%$ for IRL-P, and even $21.4\%$ when the best IRL policy is chosen. This leads to our second insight:
\begin{insight}\label{ins:inter_day}
    When executed effectively, strategic demand development via inter-period anticipation can be a powerful tool to enhance business and grow revenue in the longer run. Whether regions should be prioritized or resources should be distributed more evenly over the regions depends on the instance setting considered.
    
\end{insight}

\begin{figure}[t]
    \centering
    \captionsetup[subfigure]{width=1\linewidth}
	\begin{subfigure}[b]{\linewidth}
		\centering
		\includegraphics[width=\linewidth]{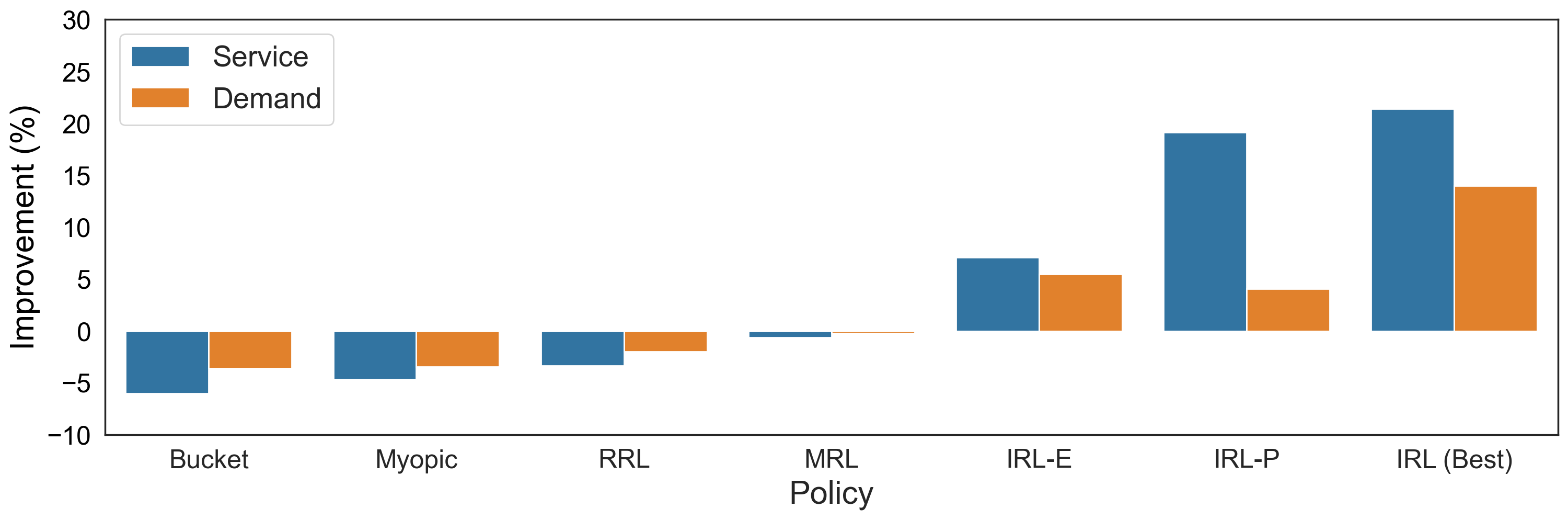}
	\end{subfigure}
	\caption{Average improvement of services provided relative to the intra-day policy over all settings.}
        \label{fig:services_value_of_shaping_by_policy}
\end{figure}



We further observe that the reward-shaping methods are inferior to the intra-day policy. Thus, at least for our problem setting, reward-shaping is not an effective means of controlling demand strategically.

Our main objective is to maximize the total number of services provided. Still, another important aspect is the expected demand at the end of the time horizon. The light (orange-colored) bars in Figure~\ref{fig:services_value_of_shaping_by_policy} show the average improvement in end-of-horizon demand. Detailed results are presented in~\ref{appendix:tables}. We observe that our IRL policies improve demand by $5.4\%$ for IRL-E, $4.0\%$ for IRL-P, and even by $14.0\%$ when the best IRL policy is chosen. This leads to the third insight:

\begin{insight}\label{ins:regiona}
    Effective strategic demand development improves not only service levels and revenue during the planning horizon but also the customer base for future operations.
\end{insight}


While both information shaping approaches lead to significant improvements, their performance varies in different individual experimental settings. We will investigate this in the following. Grouped by geography, Table~\ref{table:results} shows the average number of daily services for each policy and for the 33 experimental settings. The best value per instance is highlighted in bold. Note that Geography~(b) has the same initial demand in the two regions, so RRL is identical to the intra-day policy in that case. 


\begin{table}[t]
    \centering
    \addtolength{\tabcolsep}{0.5pt}
    \setlength\extrarowheight{-50pt}
    \begin{tabular}{@{}>{}l*{10}{l}@{}}
    \toprule[1pt]
        Demand model & $\alpha$ & $\bar{r}$ & IRL-E & IRL-P & Intra-day & Myopic & Bucket & RRL & MRL\\ 
        \midrule
        & & & \multicolumn{7}{c}{Geography (a)}\\ 
        \midrule
        Capacitated & 0.25 &- & \textbf{288.9} & 250.6 & 262.3 & 246.5 & 243.3 & 255.0 & 261.7\\ 
        & 0.5 &- & \textbf{289.9} & 251.7 & 262.4 & 246.9 & 244.1 & 254.9 & 262.2\\ 
        & 0.75 &- & \textbf{290.1} & 252.4 & 262.6 & 247.1 & 244.4 & 254.9 & 262.5\\ 
        Uncapacitated &- & 0.5 & 281.9 & \textbf{368.2} & 285.3 & 267.5 & 261.7 & 276.4 & 283.1\\ 
        &- & 0.55 & 282.9 & \textbf{366.9} & 281.3 & 263.4 & 257.1 & 270.7 & 283.1\\ 
        &- & 0.6 & 280.2 & \textbf{358.0} & 276.8 & 259.2 & 252.0 & 266.0 & 273.8\\ 
        &- & 0.65 & 278.0 & \textbf{349.0} & 272.7 & 254.4 & 246.3 & 262.2 & 269.3\\ 
        &- & 0.7 & 277.9 & \textbf{340.4} & 268.9 & 249.7 & 240.2 & 259.0 & 265.0\\ 
        &- & 0.75 & 280.5 & \textbf{331.6} & 265.8 & 245.0 & 233.7 & 255.8 & 261.2\\ 
        &- & 0.8 & 281.2 & \textbf{323.8} & 263.0 & 240.8 & 226.4 & 252.0 & 257.9\\ 
        &- & 0.85 & 277.6 & \textbf{315.7} & 260.4 & 236.6 & 217.8 & 247.7 & 254.4\\ 
        \midrule
        & & & \multicolumn{7}{c}{Geography (b)}\\ 
        \midrule  
        Capacitated & 0.25 &- & 288.3 & \textbf{290.5} & 266.1 & 248.7 & 248.7 & 266.1 & 266.09 \\
        & 0.5 &- & 289.1 & \textbf{295.8} & 266.8 & 249.5 & 249.5 & 266.8 & 266.76 \\
        & 0.75 &- & 289.2 & \textbf{297.3} & 267.0 & 249.7 & 249.7 & 267.0 & 266.95 \\
        Uncapacitated &- & 0.5 & 304.5 & \textbf{425.6} & 289.4 & 272.2 & 272.1 & 289.4 & 287.34 \\
        &- & 0.55 & 316.2 & \textbf{435.4} & 284.7 & 267.8 & 267.7 & 284.7 & 282.81 \\
        &- & 0.6 & 317.9 & \textbf{421.2} & 279.8 & 263.0 & 262.9 & 279.8 & 278.25 \\
        &- & 0.65 & 308.8 & \textbf{404.4} & 275.1 & 257.8 & 257.8 & 275.1 & 273.83 \\
        &- & 0.7 & 297.5 & \textbf{385.3} & 270.6 & 252.4 & 252.3 & 270.6 & 269.57 \\
        &- & 0.75 & 292.4 & \textbf{364.2} & 266.0 & 247.1 & 245.4 & 266.0 & 265.29 \\
        &- & 0.8 & 288.9 & \textbf{295.6} & 261.5 & 241.4 & 238.1 & 261.5 & 260.91 \\
        &- & 0.85 & \textbf{283.7} & 229.1 & 256.1 & 235.6 & 230.6 & 256.1 & 255.74 \\
        \midrule        
        & & & \multicolumn{7}{c}{Geography (c)}\\ 
        \midrule        
        Capacitated & 0.25 &- & 270.6 & \textbf{278.0} & 248.9 & 248.3 & 247.5 & 241.9 & 248.9\\ 
        & 0.5 &- & 271.0 & \textbf{278.5} & 249.0 & 248.5 & 247.7 & 242.1 & 248.8\\ 
        & 0.75 &- & 270.5 & \textbf{274.5} & 249.0 & 248.6 & 247.9 & 242.3 & 248.7\\ 
        Uncapacitated &- & 0.5 & 267.2 & 245.8 & 266.6 & \textbf{267.9} & 266.9 & 256.6 & 265.2\\ 
        &- & 0.55 & 262.5 & 259.0 & 261.4 & \textbf{263.8} & 262.7 & 252.6 & 260.0\\ 
        &- & 0.6 & 261.6 & \textbf{268.4} & 257.6 & 259.4 & 258.2 & 249.0 & 255.9\\ 
        &- & 0.65 & 265.4 & \textbf{278.9} & 254.3 & 254.6 & 253.5 & 245.6 & 252.6\\ 
        &- & 0.7 & 269.5 & \textbf{287.9} & 251.5 & 250.0 & 248.3 & 242.9 & 249.8\\ 
        &- & 0.75 & 270.8 & \textbf{296.5} & 248.7 & 245.6 & 243.4 & 240.6 & 246.7\\ 
        &- & 0.8 & 269.9 & \textbf{301.7} & 245.2 & 241.3 & 237.9 & 238.5 & 243.5\\ 
        &- & 0.85 & 267.1 & \textbf{300.2} & 241.5 & 236.9 & 230.7 & 236.0 & 240.2\\ 
        \bottomrule[1pt]
    \end{tabular}
    \caption{Average number of services provided per day and policy.}
    \label{table:results}
\end{table}

We observe that except in two cases, the best policy is either IRL-E or IRL-P but their performance differs with respect to geography and demand development. For Geography~(a), there is a clear distinction between capacitated and uncapacitated demand. For the capacitated, policy IRL-E clearly outperforms IRL-P while, for uncapacitated demand, the opposite is the case. This leads to the following insight:

\begin{insight}\label{ins:regiona_}
    When the geography of two regions is similar and demand per region is limited, it is worth investing in all regions even if initial demand in one region is small. If demand per region is unlimited, it is beneficial to focus on the region with high initial demand.
\end{insight}

For Geography~(b), we observe a different picture. Here, IRL-P is nearly always substantially better than IRL-E, especially for instances with uncapacitated demand developments. The favorable location of Region~1 allows for more efficient service of requests to this region. Thus, demand in this region can be grown quickly and, at the same time, can be fulfilled more efficiently. There is one exception in the uncapacitated demand model when customer expectations are very high, $\bar{r}=0.85$. In this case, demand does not grow as much as others, and using the fleet to serve the existing demand in both regions via IRL-E becomes superior.

\begin{insight}\label{ins:regionb}
The spatial distribution of customers plays an important role in the decision about how to invest fleet resources. Here, it is usually beneficial to focus on conveniently located regions rather than spending significant resources satisfying demand in regions that require longer travel.
\end{insight}

For Geography~(c), even though the regions and their overall demand potential are identical, policy IRL-E does not provide the best results. Policy IRL-P performs slightly better. The reason is that, by spreading resources over four regions, demand per region does not increase as fast as in Geography~(a). However, over a horizon longer than two years, IRL-E may become superior eventually.

\begin{insight}\label{ins:regionc}
    When strategically distributing resources, the considered time horizon plays an important role. A shorter horizon may lead to exploiting existing demand structures, while a longer horizon may encourage investment in regions with lower initial demand.
\end{insight}


\subsection{Illustration of Intra-Day and Inter-Day Anticipation}

In the following, we illustrate how different policies with and without intra- and/or inter-day anticipation impact the development of demand in different regions. We use Geography~(a) with uncapacitated demand and $\bar{r}=0.8$ as an example. Figure~\ref{fig:inter-day} presents the average demand curves for the two regions and for the following five different policies:
\begin{enumerate}
    \item [{(1)}] Myopic, anticipating neither intra-day nor inter-day developments (average daily services 240.8).
    \item [{(2)}] The intra-day strategy that anticipates only intra-day developments (263.0).
    \item [{(3)}] The bucket policy that anticipates only inter-day developments (226.4).
    \item [{(4)}] The IRL-P policy that anticipates both (323.8).  
    \item [{(5)}] The IRL-E policy that anticipates both (281.2). 
\end{enumerate}

The development of the myopic policy shows a slight increase in demand at the beginning for both regions. A short time later, demand in Region~1 starts decreasing significantly while demand in Region~2 increases. Eventually, they reach similar levels around 150 requests per day. This behavior can be expected. The myopic strategy serves all feasible requests. Thus, it treats both regions similarly. At the same time, decision making is not anticipatory, and thus daily resources are consumed quickly. As a result, the overall service level is low.

For the intra-day policy, we see a different picture. The policy learns that requests from Region~1 are more valuable. Consequently, demand in Region~1 remains at a higher level while demand in Region~2 only grows slightly. This leads to a final total demand of about 325 requests, significantly higher than the myopic strategy. The bucket policy aims to balance demand between the two regions. This can also be seen in the demand development. The demand ratios for both regions converge quite quickly, reaching about 150 each after one year. We also observe that, in contrast to the myopic policy, demand in Region~1 drops initially. This is because requests in Region~1, which could have been fulfilled efficiently, are rejected to ensure equal treatment of both regions. Therefore, the bucket policy performs worst.

\begin{insight}
    Forcing equal service levels in all regions without considering the effectiveness of daily decision-making does not lead to market growth. Service in high-demand regions is neglected and customers leave.
\end{insight}

\begin{figure}[t]
\centering
\includegraphics[width=0.9\textwidth]{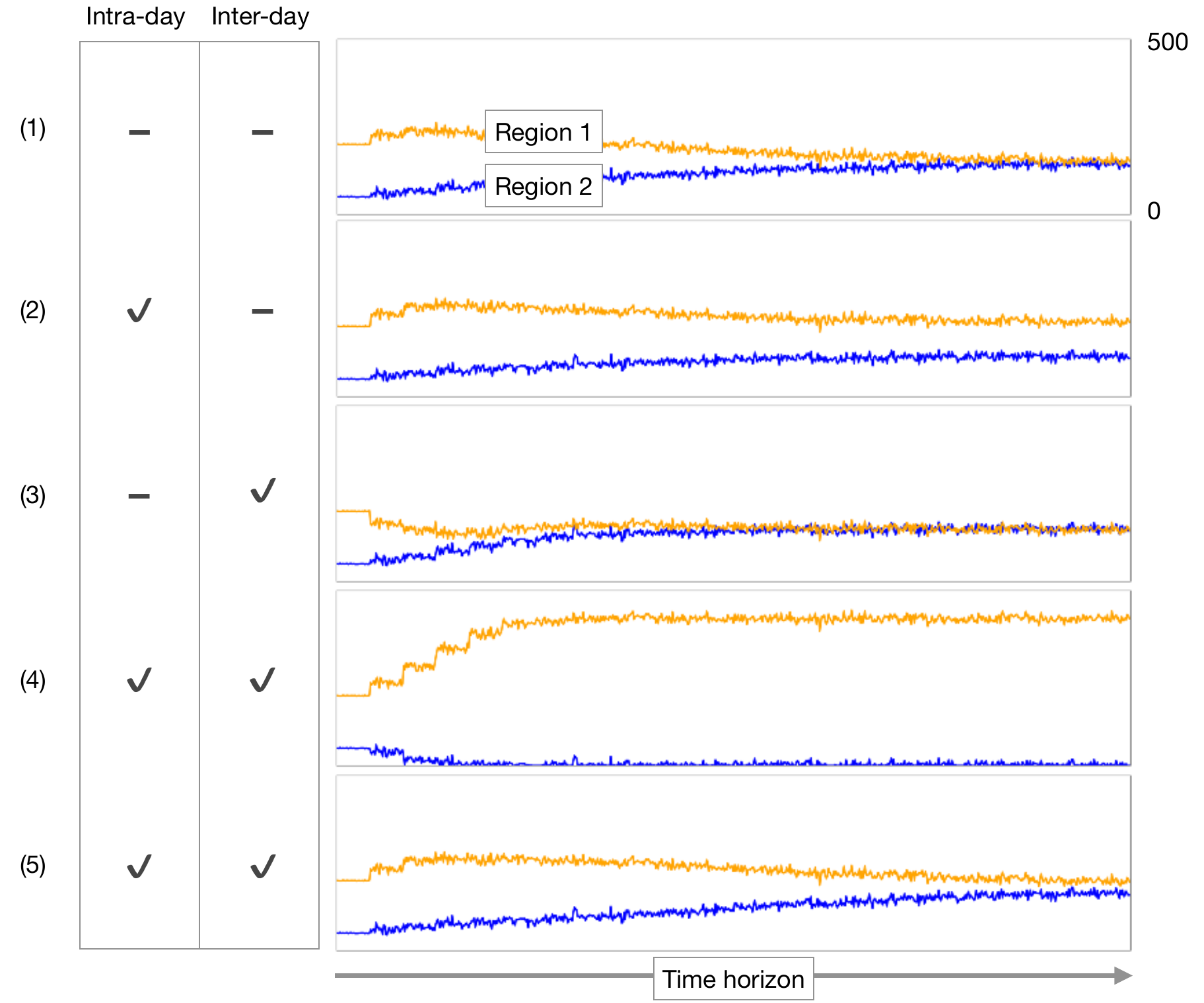}
\caption{Example for demand development when applying policies with and without intra-day and inter-day anticipation.}
\label{fig:inter-day}
\end{figure}

IRL-P behaves differently from all other policies in this illustration. Demand in Region~1 increases fast while demand in Region~2 drops to nearly zero, very similar to the development of Strategy~A in Figure~\ref{fig:example_inter}. This policy systematically prioritizes Region~1, which results in very high demand in Region~1 (around 400) and no demand in Region~2.

Compared to IRL-P, IRL-E presents a rather opposite behavior, similar to Strategy B in Figure~\ref{fig:example_inter}. Demand in Region~1 increases at the beginning because delivery resources are available given the limited demand in Region~2. Then, demand in Region~2 increases fast while demand in Region~1 converges to 200 requests per day. The same value is nearly obtained in Region~2, but the two-year horizon is not long enough to see both regions reach the same value. However, the final total demand is again close to 400, a little lower than IRL-P. Therefore, even though the objective value of IRL-E is significantly lower compared to IRL-P, the final customer base is more balanced. This leads us to our final insight:

\begin{insight}\label{ins:final_demand}
    Compared to treating all regions equally, prioritizing regions might lead to more revenue over the planning horizon, but it also results in a very unbalanced demand structure in the end. Companies therefore must acknowledge the trade-off between making money in the first years and establishing a customer base in the entire city.
\end{insight}

\section{Conclusions and Future Work}\label{sec:conclusions}



In this paper, we address a challenge faced by many on-demand delivery start-ups: how to manage long-term demand growth with limited delivery resources in the initial stages of the business. We developed an intra-day model to capture the service provider's daily operations and an inter-day model to account for the evolution of customer demand in each neighborhood. The objective was to maximize the expected number of services offered in the long run. To gain analytical insights into the optimal solutions, we studied a stylized inter-day problem, which revealed two strategies for delivery resource allocation. Inspired by the analytical results, we proposed four strategies for using reinforcement learning with information shaping to learn the intra-day policy while maximizing the inter-day objective. 

We conducted comprehensive experimental studies with two different demand models. The results demonstrate the effectiveness of our proposed training strategies and validate the two resource allocation strategies suggested by the properties of the optimal solutions. We find that in neighborhoods with a low potential for demand growth, it performs better to distribute delivery resources equally to all neighborhoods to maximize the total number of services provided. On the contrary, if neighborhoods have a high potential for demand growth, to achieve the maximum service levels, the service provider should invest all delivery resources in the initially high-demand neighborhood. We also show the economic value of incorporating demand development when designing the algorithm. Further, since the algorithm used is the core of on-demand delivery service providers' business, our proposed training strategies also make valuable contributions. 

This work offers several directions for future research. First, while efficiency remains the primary objective of any business, experiments in this paper show that policies differ in not only the number of services but also the demand distribution in the end. Here, future work may investigate the value of different distributions. Furthermore, the question of fairness within operations remains a topic worth exploring. Other work might consider location planning or service design in combination with our work. An example is to consider prices or service promise as discussed in \cite{stroh2022tactical}. Specifically, tight delivery deadlines and cheap prices may increase demand, but they can also reduce the number of services that can be made and the revenue generated from deliveries. Another interesting extension related to \cite{luy2023strategic} would be the option of adding more vehicle resources over time, e.g., after one year. Further, in our experiments, we assume that we know the existence of the expected demand development (but do not know how), while, in practice, this might be uncertain as well. Thus, predictions for expected demand and demand developments might be valuable. A potential starting point could be the work by \cite{zinnenlauf2024value}. Finally, we study an on-demand delivery problem, and it is worth exploring other types of on-demand transportation services, such as on-demand pick-up and delivery, meal delivery, or ride-sharing.  

Another area of interesting research opportunities is information shaping for reinforcement learning. To the best of our knowledge, we are the first to use analytical problem insights to adapt training data for RL. While our information shaping approach focuses on the stochastic information, other methods might shape the transition function instead, such as by adding or removing state information. Future work may further investigate when and how information shaping can be of particular value. Application areas are not limited to transportation and logistics, but may apply to other domains where complex decisions must be made under uncertainty, such as production and healthcare operations.



\bibliographystyle{abbrvnat}

\appendix

\section{Proofs}
\subsection{Theorem~\ref{theorem:optimal_main_body}}\label{appendix:theorem1}

    First, consider the objective as a function of $r_2$ and $\lambda_2$, denoted by $f(r_2,\lambda_2)$. With Proposition~\ref{binding_assumption}, solving the resource constraint for $r_1\lambda_1$ yields 
    \begin{equation}\label{ri_lambdai}
    r_1\lambda_1 = \frac{T^2-2T\beta\sqrt{Ar_2\lambda_2}+A\beta^2r_2\lambda_2}{A\beta^2}.
    \end{equation}
    Substituting $r_1\lambda_1$ in the objective according to Equation~\ref{ri_lambdai}, we get
    \begin{equation}
        f(r_2,\lambda_2)=\frac{T^2-2T\beta\sqrt{Ar_2\lambda_2}+A\beta^2r_2\lambda_2}{A\beta^2}+r_2\lambda_2.
    \end{equation}
    Since $r_2$ and $\lambda_2$ are constrained by $\lambda_2 = log(M_2r_2+1)$, substituting $\lambda_2$ in $f(r_2,\lambda_2)$ yields
    \begin{equation}
        f(r_2)=\frac{T^2-2T\beta \sqrt{A log(M_2 r_2+1)r_2}+2A\beta^2log(M_2r_2+1)r_2}{A\beta^2},\ 0\leq r_2\leq 1.
    \end{equation}
    Differentiating $f(r_2)$ with respect to $r_2$, we obtain the derivative of the objective function
    \begin{equation*}
        f'(r_2)=\frac{[log(M_2r_2+1)+M_2r_2+log(M_2r_2+1)M_2r_2][2\beta \sqrt{Alog(M_2r_2+1)r_2}-T]}{\beta(1+M_2r_2) \sqrt{Alog(M_2r_2+1)r_2}}
        +log(M_2r_2+1)
        +\frac{M_2r_2}{M_2r_2+1}
    \end{equation*}
    \begin{equation}\label{obj_derivative}
        =\frac{[log(M_2r_2+1)+M_2r_2+log(M_2r_2+1)M_2r_2]\cdot[2\beta\sqrt{A log(M_2r_2+1)r_2}-T]}
        {\beta \sqrt{Alog(M_2r_2+1)r_2}(M_2r_2+1)}.
    \end{equation}
    Note that the objective function $f(r_2)$ is defined over $r_2\in[0,1]$ but the derivative $f'(r_2)$ is not defined at $r_2=0$. Extremas can occur where $f'(r_2)$ equals zero or where $f(r_2)$ is defined but $f'(r_2)$ is not defined. To this end, we then consider two subsets of $0\leq r_2\leq 1$, i.e., $0< r_2\leq 1$ (Lemma~\ref{lemma:main_excluding_0}) and $r_2=0$. The findings are summarized in Theorem~\ref{theorem:optimal_main_body}.

    \begin{proof}{Proof of Lemma~\ref{lemma:main_excluding_0}.}
    
    The denominator in Equation~\ref{obj_derivative} is nonnegative for all $0<r_2\leq 1$. Setting the numerator equal to $0$ to find any critical points, we get 
    \begin{equation}\label{first_critical}
        log(M_2r_2+1)+M_2r_2+log(M_2r_2+1)M_2r_2=0.
    \end{equation}
    \begin{equation}\label{second_critical}
        \textrm{or\ \ \ }\ 2\beta\sqrt{A log(M_2r_2+1)r_2}-T=0.
    \end{equation}
    Equation~\ref{first_critical} has no solution because all three terms on the left-hand side are strictly greater than $0$ when $0<r_2\leq 1$. Rearranging the terms in Equation~\ref{second_critical} gives us
    \begin{equation}\label{half_resource}
        \beta\sqrt{A log(M_2r_2+1)r_2}=\frac{T}{2}.
    \end{equation}
    Because of the demand equation $\lambda_2=log(M_2r_2+1)$, Equation~\ref{half_resource} becomes
    \begin{equation}\label{half_resource_2}
        \beta\sqrt{A\lambda_2 r_2}=\frac{T}{2}.
    \end{equation}    
    The left-hand side is delivery resources the service provider invests in region $\mathcal{Z}_2$ according to the Beardwood, Halton, and Hammersley approximation. The right-hand side is exactly half of the total delivery resources available to the service provider. We are aware that we need to further analyze if there is a global maxima when Equation~\ref{half_resource} is achieved at $r\textsubscript{critical}$. However, due to the lack of the closed-form solution to this equation, it is not very practical to do so. (We conduct comprehensive computations to validate the findings.) Consequently, this Lemma can be considered as potentially optimal in addressing the problem.
    \end{proof}

    Next, we include $r_2=0$ in our analysis. When $r_2=0$, the service provider invests all delivery resources $T$ in the other region $\mathcal{Z}_1$, not necessarily accepting all requests. (Note that the TSP approximation $\beta\sqrt{A\lambda_1 r_1}$ represents the delivery resources invested, e.g., time, while $\lambda_1 r_1$ represents the actual number of services offered.) At $r_2=0$, the objective value is $f(0)=\frac{T^2}{A\beta^2}$. This objective value can be the maximum if $f(0)>f(r\textsubscript{critical})$. Again, because the exact value of $r\textsubscript{critical}$ is lacking, it is only possible that the optimal solution is $r_2=0$.

    Finally, the above analyses complete the proof for Theorem~\ref{theorem:optimal_main_body}.

\subsection{Corollary~\ref{optimal_ratio_theorem_main_body}}\label{sec:corollary_proof}
    \begin{proof}[Proof of Corollary~\ref{optimal_ratio_theorem_main_body}.]
    Assume the optimal solution $(r^*_1,r^*_2)$ occurs when Equation~\ref{second_critical} is true. Solving Equation~\ref{second_critical} yields 
    \begin{equation}\label{j_critical}
        log(M_2r^*_2+1)r^*_2=\frac{T^2}{4A\beta^2}.
    \end{equation}
    If we replace the index $2$ by $1$ in the beginning and repeat the same steps from Equation~\ref{ri_lambdai} to~\ref{second_critical}, we can obtain a similar equation,
    \begin{equation}\label{i_critical}
        log(M_1r^*_1+1)r^*_1=\frac{T^2}{4A\beta^2}.
    \end{equation}
    Then,
    \begin{equation}
        log(M_2r^*_2+1)r^*_2=\frac{T^2}{4A\beta^2}=log(M_1r^*_1+1)r^*_1,
    \end{equation}
    \begin{equation}
        log(M_2r^*_2+1)r^*_2=log(M_1r^*_1+1)r^*_1.
    \end{equation}
    Since $\lambda_i = log(M_ir^*_i+1)$ for $i=1,2$, we get
    \begin{equation}
        \lambda_2 r^*_2=\lambda_1 r^*_1,
    \end{equation}
    \begin{equation}
        \frac{r^*_1}{r^*_2} =\frac{\lambda_2}{\lambda_1}.
    \end{equation}
    \end{proof}

\section{Details for Reinforcement Learning via Information Shaping}\label{sec:additional_details}

\subsection{Intra-day Solution}\label{appendix:intra_day}
Given how information shaping controls demand developments from day to day, we next describe how RL is used to train the policy for each day. Our method uses an artificial neural network, making the approach so-called ``deep reinforcement learning''. The use of the neural network eliminates the need for knowledge of the exact form of the value function. However, the large decision space still prevents neural networks from learning efficiently. To this end, rather than having the neural net learning routing in addition to whether to accept a customer request, we reduce the decision space by implementing a \textit{cheapest insertion} routing heuristic. This heuristic inserts a customer into a vehicle's delivery route at the position that results in the smallest increase in travel time (see \cite{azi} for details). The use of the routing heuristic reduces the number of possible decisions to only $P+1$ (assignment to each vehicle plus no service), and the resulting reduced decision space at the \textit{k}th decision point is $\hat{\mathcal{X}}_k^\iota=\{0,1,\dots,P\}$. Upon receiving a customer request, the service provider can accept the request and assign it to a vehicle $v_p$ (decision $p$), or not offer the service at all (decision $0$). Note that if a vehicle $v_p$ cannot feasibly make the delivery to this customer before the deadline, then decision $p$ is not available at this decision point. Moreover, even if all vehicles can feasibly make the delivery, the provider can still decide not to offer the service. 

With the reduced decision space, our deep reinforcement learning approach learns the expected value of state-decision pairs. This approach is referred to as ``deep Q-learning'', and the state-decision value is referred to as ``Q-value.'' The Q-value of decision $x_k$ in state $S_k$ is defined by the Q-function
\begin{equation}
    Q(S_k^\iota, x_k^\iota)=R^\iota(S_k^\iota,x_k^\iota) + \mathbb{E}[\hat{V}(S_{k+1}^\iota)|S_k^\iota,x_k^\iota].
\end{equation}
The hat notation indicates that $\hat{V}(S_{k+1}^\iota)$ is the value of state $S_{k+1}^\iota$ given the reduced decision space. Then, instead of solving it exactly, we can solve an approximate form of the Bellman equation,
\begin{equation}\label{eq:qBellman}
\hat{V}(S_k^\iota)=\max_{x_k^\iota\in \hat{\mathcal{X}_k^\iota}}\{Q(S_k^\iota, x_k^\iota)\}.
\end{equation}

To use the vast amount of information in our problem, we employ an artificial neural network denoted by a set of weights. The input layer takes the information from the state, referred to as ``features'', and the output layer provides the approximated Q-values. The features we consider are as follows. 
\begin{itemize}
    \item \textit{Decision point}: (1) Time of decision point, $t_k^\iota$. 
    \item \textit{Customer}: (2) The region from which the request is made. (3) Proxy distance to the warehouse, represented by a vehicle's direct travel time from the warehouse to the customer.
    \item \textit{Vehicles}: (4) Time at which each vehicle returns to the warehouse. (5) Each vehicle's feasibility to serve this customer. (6) Each vehicle's increase in travel time if the current customer is assigned to it (a large number when infeasible).
    \item \textit{Region}: (7) Expected demand of each region. (8) Current service level in the day of each region. 
\end{itemize}

These features have different ranges of values. To prevent certain broad-range features from dominating the learning, we apply min-max normalization to them so that every feature ranges from $0$ to $1$ \citep{bishop1995neural}. We perform training on the weights $\phi$ of the neural net. Thus, the policy that uses the learned neural net to make decisions is a solution to the intra-day problem.

\subsection{Details for Training and Evaluation}\label{appendix:training_details}

The neural network is composed of two hidden layers, each consisting of $50$ neurons. The activation function used in both the input and hidden layers is the ReLU function, denoted as, $ReLU(x)=\max\{x,0\}$. For each demand setting, we generate a training set of $1500$ instances and a test set of $500$ instances. During training, we utilize an $\epsilon$-greedy strategy to select actions, where the value of $\epsilon$ decreases exponentially from $1$ to $0.01$. To address the issue of correlations in consecutive states, we implement experience replay, as proposed in \citet{lin}. The training concludes after completing $200000$ epochs. 
Then, we evaluate the policy $100$ times and report the average performance.

In our experiments, to get the historical service level for each region, we divide the total number
of services in a region by the total number of requests received from that region in the last $30$ days. Specifically, at the $u$th update, the historical service level for Region~$i$ in the last $30$ days is calculated as 
$$r_{u-1,i}=\frac{\sum_{l=0}^{l=29}\sum_{k=0}^{K}\mathbbm{1}_{i}(c_{30u-l,k})\cdot a_{30u-l,k}}{\sum_{l=0}^{l=29}|\mathcal{C}_{30u-l,i}|}.$$

The numerator calculates the number of services in Region~$i$. The indicator function $\mathbbm{1}_{i}(\cdot)$ checks if a customer request is made from Region~$i$. The notation $\mathcal{C}_{30u-l,i}$ in the denominator represents the set of customer requests made on day $30u-l$ from Region~$i$, and thus the denominator calculates the total number of requests made from Region~$i$.

In Section~\ref{sec:information_shaping}, we introduce the approach of using a coefficient of variation (COV) of $0.5$. In our experiments, we explored several values of COV for the IRL-P. In the main paper, we report the results of the IRL-P that was trained with a COV of $0.25$ for the prioritized region(s) and a COV of $0.5$ for the other region(s). The performance of this policy shows a smaller variance in the number of requests served over different geographies. 

\section{Additional Details for Benchmark Policies}\label{appendix:benchmark}
This section provides the details for decision-making of benchmark policies. We use the notations from Section~\ref{sec:intra_day} and focus on the acceptance decision in the process. When the policy involves the demand evolution, we append a $\tau_n$ in the indices. 

\subsection{Myopic policy.} In a state $S_k^\iota$, let binary $a_{p,k}$ represent if vehicle $p$ can feasibly serve the customer $c_k$. Then, the acceptance decision $a_k$ is defined as 
$$ a_k=\left\{
\begin{array}{rcl}
0,       &      & {\textrm{if }} \sum_{p=1}^{P} a_{p,k} =0,\vspace{0.2cm}\\
1,     &      & {\textrm{if }} \sum_{p=1}^{P} a_{p,k} \geq 1.
\end{array} \right. $$
That is, if there are no vehicles that can feasibly serve the customer, the customer is not offered the service. If there there is at least one vehicle that can feasibly serve the customer, we accept the request. Note that when there are more than one vehicle that can offer the service, we select the one that has the minimum increase in tour time.
\subsection{Intra-day policy.} See~\ref{appendix:intra_day} and \cite{chen2023same} for details of the policy. 
\subsection{Bucket policy.} 
Assume on day $\tau_m$, the expected demands for the $I$ regions are $D_{\tau_m,1},\dots,D_{\tau_m,I}$. In a state $S_{\tau_m,k}^\iota$, the number of requests that have been accepted from Region~$i$ until $S_{\tau_m,k}^\iota$ is calculated as 
$$\mu_{\tau_m,i}=\sum_{k=0}^{K}\mathbbm{1}_{i}(c_{\tau_m,k})\cdot a_{\tau_m,k}.$$
Let $r(c_{\tau_m,k})$ represent the region from which the request is made. Then, the acceptance decision is defined as 
$$ a_{\tau_m,k}=\left\{
\begin{array}{rcl}
0,       &      & {\textrm{if }} \mu_{\tau_m,r(c_{\tau_m,k})} \geq\frac{D_{\tau_m,1}+\dots+D_{\tau_m,I}}{I},\vspace{0.2cm}\\
1,     &      & {\textrm{if }} \mu_{\tau_m,r(c_{\tau_m,k})} <\frac{D_{\tau_m,1}+\dots+D_{\tau_m,I}}{I}.
\end{array} \right. $$

\subsection{Reward-shaping policy (RRL).}
This policy ignores demand evolution during training, i.e., using initial or historical demands all the time. Assume the expected demands for the regions are $D_{1,1},\dots,D_{1,I}$. During training, the reward of a decision $x_k^\iota=(a_k,\Theta_k)$ in state $S_k^\iota$ is
\begin{equation}
R^\iota(S_k^\iota,x_k^\iota)=\left\{
\begin{array}{rcl}
\frac{max\{D_{1,1},\dots,D_{1,I}\}}{\sum_{i=1}^{I}\mathbbm{1}_{i}(c_{k})\cdot D_{1,i}},       &      & {\textrm{if $a_k=1$,}}\vspace{0.2cm}\\
0,     &      & {\textrm{if $a_k=0$.}}
\end{array} \right.
\end{equation}
The numerator of the fraction finds the maximum expected demand over all regions. The denominator calculates the expected demand for the region that customer $c_k$ resides in. 

\subsection{Manipulated-reward policy (MRL).}
The training of this policy is the same as that for the intra-day policy. However, during the implementation of the policy, we manipulate the reward of customers. Assume on day $\tau_m$, the expected demands for the $I$ regions are $D_{\tau_m,1},\dots,D_{\tau_m,I}$, respectively. During implementation, the reward is calculated as 
$$R^\iota(S_{\tau_m,k}^\iota,x_{\tau_m,k}^\iota)=\left\{
\begin{array}{rcl}
\frac{max\{D_{\tau_m,1},\dots,D_{\tau_m,I}\}}{\sum_{i=1}^{I}\mathbbm{1}_{i}(c_{\tau_m,k})\cdot D_{\tau_m,i}},       &      & {\textrm{if $a_{\tau_m,k}=1$,}}\vspace{0.2cm}\\
0,     &      & {\textrm{if $a_{\tau_m,k}=0$.}}
\end{array} \right.$$
\section{Additional Experimental Results: Detailed End-of-Year Demand Values}\label{appendix:tables}

This section presents the detailed demand values at the end of the two-year horizon. 

\begin{table}[h]
    \centering
    \addtolength{\tabcolsep}{1pt}
    \setlength\extrarowheight{-50pt}
    \begin{tabular}{@{}>{}l*{10}{l}@{}}
    \toprule[1pt]
        Demand model & $\alpha$ & $\bar{r}$ & IRL-E & IRL-P & Intra-day & Myopic & Bucket & RRL & MRL\\ 
        \midrule
        & & & \multicolumn{7}{c}{Geography (a)}\\ 
        \midrule
        Capacitated & 0.25 & - & \textbf{381.5} & 289.6 & 361.4 & 350.7 & 350.8 & 356.4 & 361.3 \\ 
        &0.5 & - & \textbf{380.6} & 287.8 & 360.1 & 349.8 & 349.9 & 355.6 & 360.7 \\ 
        &0.75 & - & \textbf{380.3} & 286.6 & 359.9 & 348.8 & 349.8 & 354.6 & 360.0 \\ 
        Uncapacitated &- & 0.5 & 568.5 & \textbf{763.7} & 575.1 & 536.7 & 537.1 & 554.1 & 572.9 \\ 
        &- & 0.55 & 517.8 & \textbf{697.1} & 513.4 & 479.4 & 479.3 & 491.5 & 510.7 \\ 
        &- & 0.6 & 469.3 & \textbf{623.3} & 461.6 & 431.3 & 431.4 & 441.6 & 459.5 \\ 
        &- & 0.65 & 430.3 & \textbf{560.2} & 418.8 & 389.6 & 389.5 & 400.1 & 415.5 \\ 
        &- & 0.7 & 400.5 & \textbf{507.4} & 382.6 & 353.5 & 354.0 & 366.3 & 378.1 \\ 
        &- & 0.75 & 383.4 & \textbf{461.5} & 351.5 & 322.5 & 323.6 & 337.0 & 347.2 \\ 
        &- & 0.8 & 361.1 & \textbf{423.7} & 325.7 & 295.6 & 295.6 & 310.2 & 320.0 \\ 
        &- & 0.85 & 334.7 & \textbf{390.9} & 302.6 & 272.1 & 271.8 & 285.5 & 295.6 \\ 
        \midrule
        & & & \multicolumn{7}{c}{Geography (b)}\\ 
        \midrule  
        Capacitated & 0.25 & - & \textbf{382.3} & 379.1 & 367.0 & 354.9 & 355.1 & 367.0 & 367.0 \\
        &0.5 & - & \textbf{381.9} & 378.9 & 366.7 & 354.6 & 354.0 & 366.7 & 366.5 \\ 
        &0.75 & - & \textbf{382.3} & 378.4 & 367.1 & 355.2 & 354.4 & 367.1 & 367.0 \\ 
        Uncapacitated &- & 0.5 & \textbf{525.7} & 412.9 & 489.0 & 491.9 & 492.2 & 489.0 & 485.4 \\ 
        &- & 0.55 & \textbf{470.6} & 367.3 & 443.6 & 439.9 & 440.0 & 443.6 & 441.7 \\ 
        &- & 0.6 & \textbf{425.1} & 307.9 & 407.2 & 396.7 & 396.8 & 407.2 & 407.6 \\ 
        &- & 0.65 & \textbf{391.8} & 267.9 & 378.1 & 361.5 & 361.2 & 378.1 & 378.4 \\ 
        &- & 0.7 & \textbf{366.7} & 249.8 & 353.3 & 332.2 & 331.8 & 353.3 & 354.0 \\ 
        &- & 0.75 & \textbf{345.4} & 239.4 & 331.7 & 308.2 & 307.8 & 331.7 & 332.9 \\ 
        &- & 0.8 & \textbf{317.2} & 235.0 & 312.5 & 288.5 & 287.9 & 312.5 & 314.5 \\ 
        &- & 0.85 & \textbf{330.3} & 233.2 & 295.4 & 271.6 & 271.2 & 295.4 & 298.1 \\ 
        \midrule        
        & & & \multicolumn{7}{c}{Geography (c)}\\ 
        \midrule        
        Capacitated & 0.25 & - & \textbf{365.8} & 359.9 & 349.9 & 350.0 & 349.9 & 344.8 & 347.2 \\ 
        &0.5 & - & \textbf{363.1} & 358.3 & 347.5 & 347.7 & 347.7 & 342.8 & 345.5 \\ 
        &0.75 & - & \textbf{361.5} & 355.5 & 345.9 & 346.6 & 346.5 & 341.2 & 349.6 \\ 
        Uncapacitated &- & 0.5 & 531.1 & 479.9 & 531.0 & \textbf{536.0} & 535.6 & 513.2 & 531.0 \\ 
        &- & 0.55 & 471.4 & 463.8 & 470.2 & \textbf{477.8} & \textbf{477.8} & 456.3 & 470.2 \\ 
        &- & 0.6 & 430.9 & \textbf{441.4} & 422.5 & 429.4 & 429.1 & 410.4 & 422.5 \\ 
        &- & 0.65 & 405.2 & \textbf{422.2} & 384.1 & 387.4 & 387.9 & 373.4 & 384.1 \\ 
        &- & 0.7 & 382.8 & \textbf{403.9} & 350.5 & 351.9 & 350.6 & 340.7 & 350.5 \\ 
        &- & 0.75 & 359.2 & \textbf{396.9} & 322.4 & 321.3 & 319.2 & 313.8 & 322.4 \\ 
        &- & 0.8 & 334.8 & \textbf{384.6} & 296.9 & 294.2 & 290.4 & 290.6 & 296.9 \\ 
        &- & 0.85 & 310.9 & \textbf{364.1} & 274.5 & 270.7 & 261.2 & 270.0 & 274.5 \\ 
        \bottomrule[1pt]
    \end{tabular}
    \caption{Average number of final expected demand policy.}
    \label{table:results_demand}
\end{table}

\end{document}